\documentclass[11pt]{amsart}

\usepackage{verbatim, amssymb,hyperref}

\usepackage{color}
\usepackage{bbm}
\usepackage{graphicx}

\newcommand{\eps}{\varepsilon}

\newcommand{\cF}{\mathcal F}

\newcommand{\cC}{\mathcal C}
\newcommand{\cL}{\mathcal L}

\newcommand{\cE}{\mathcal E}

\newcommand{\cR}{\mathcal R}

\newcommand{\cM}{\mathcal M}

\newcommand{\cO}{\mathcal O}

\newcommand{\cX}{\mathcal X}

\newcommand{\dd}{\mathrm{d}}
\newcommand{\sfrac}[2]{\mbox{$\frac{#1}{#2}$}}

\newcommand{\loja}{\L ojasiewicz}
\newcommand{\CL}{\L}

\newcommand{\IB}{\mathbb B}

\newcommand{\IG}{\mathbb G}

\newcommand{\1}{1\hspace{-0.098cm}\mathrm{l}}

\renewcommand{\P}{{\mathbb P}}

\newcommand{\N}{{\mathbb N}}
\newcommand{\IA}{{\mathbb A}}
\newcommand{\IL}{{\mathbb L}}
\newcommand{\IC}{{\mathbb C}}
\newcommand{\E}{{\mathbb E}}
\newcommand{\IM}{{\mathbb M}}

\newcommand{\R}{{\mathbb R}}

\newcommand{\IU}{{\mathbb U}}
\newcommand{\Hess}{\text{Hess}\,}

\newcommand{\cP}{{\mathcal P}}

\theoremstyle{plain}
\newtheorem{theorem}{Theorem}[section]
\newtheorem{prop}[theorem]{Proposition}
\newtheorem{lemma}[theorem]{Lemma}

\newtheorem{defi}[theorem]{Definition}

\theoremstyle{definition}
\newtheorem{rem}[theorem]{Remark}

\begin{document}
		
	\title[Convergence of SGD for \L ojasiewicz-landscapes]%
	{Convergence of stochastic gradient descent schemes\\ for \L ojasiewicz-landscapes}
	
	\author[]%[Dereich]
	{Steffen Dereich}
	\address{Steffen Dereich\\
		Institut f\"ur Mathematische Statistik\\
		Fachbereich 10: Mathematik und Informatik\\
		Westf\"alische Wilhelms-Universit\"at M\"unster\\
		Orl\'eans-Ring 10\\
		48149 M\"unster\\
		Germany}
	\email{steffen.dereich@wwu.de}
	
	\author[]%[Kassing]
	{Sebastian Kassing}
	\address{Sebastian Kassing\\
		Fakult\"at f\"ur Mathematik\\
		Universit\"at Bielefeld\\
		Universit\"atsstraße 25\\
		33615 Bielefeld\\
		Germany}
	\email{skassing@math.uni-bielefeld.de}
	

\begin{abstract}In this article, we consider convergence of stochastic gradient descent schemes (SGD), including momentum stochastic gradient descent (MSGD), under weak assumptions on the underlying landscape. More explicitly, we show that on the event that the SGD stays bounded we have convergence of the SGD if there is only a countable number of critical points or if the objective function satisfies  \loja-inequalities around all critical levels as all analytic functions do. In particular, we show that for neural networks with analytic activation function such as softplus, sigmoid and the hyperbolic tangent, SGD converges on the event of staying bounded, if the random variables modelling the signal and response in the training are compactly supported.
\end{abstract}

\maketitle

\section{Introduction}
In this article, we analyse stochastic gradient descent schemes for $C^1$-objective functions $f:\R^d\to \R$ with $d\in\N:=\{1,2,\dots\}$ being an arbitrary dimension. We denote by $(\Omega,\cF,(\cF_n)_{n\in\N_0},\P)$ a filtered probability space and consider an $\R^d$-valued stochastic process $(X_n)_{n\in\N_0}$  that admits a representation 
\begin{align} \label{eq:main}
	X_n=X_{n-1} + \gamma_n( \Gamma_{n}+D_n),
\end{align}
for $n\in\N$, where 
\begin{itemize}
	\item $(\gamma_n)_{n\in\N}$ is a sequence of strictly positive reals, the \emph{step-sizes} or \emph{learning rates},
	{\item $(\Gamma_{n})_{n \in \N}$ is an $(\mathcal F_n)_{n \in \N_0}$-predictable sequence of random variables, the \emph{drift},
	}
	\item $(D_n)_{n\in\N}$ is an $(\cF_n)_{n\in\N}$-adapted sequence of random variables, the \emph{perturbation}, and
	\item $X_0$ is an $\cF_0$-measurable random variable, the \emph{initial value}.
\end{itemize}
The choice $(\Gamma_n)_{n \in \N}= (-\nabla f(X_{n-1}))_{n \in \N}$ leads to a standard representation of stochastic gradient descent. However, our results hold for a more general class of dynamical systems, including momentum stochastic gradient descent (see Section~\ref{sec:MSGD}), assuming the drift is comparable in size and direction to the gradient vector field. (The precise condition is given in Definition~\ref{def:events}.)
Additional assumptions will be imposed in the theorems below.

Stochastic gradient descent schemes form a subclass of Robbins--Monro schemes which were introduced in 1951~\cite{RM51} and have been highly influential since then. Their relevance stems from their applicability of finding zeros of functions $F:\R^d \to \R^d$ in the case where one has only simulations at hand which give approximations to the value of $F$ in mean.
Following the original papers a variety of results were derived and we refer the reader to the  mathematical accounts \cite{BMP90,Duf96,KY03} on stochastic approximation methods. 

In this article, we analyse convergence of $(X_n)_{n \in \N_0}$. This problem is intimately related to understanding  the asymptotic  behaviour of $(\nabla f(X_n))_{n \in \N_0}$.
The classical   analysis of Polyak and Tsypkin~\cite{PolTsy73} yields existence of the limit $\lim_{n\to\infty} f(X_n)$ and proves $\liminf_{n\to\infty} |\nabla f(X_n)|=0$ under appropriate assumptions.  Later, Walk~\cite{Walk92} showed that in an appropriate setting one has almost sure convergence    $\lim_{n\to\infty} \nabla f(X_n)=0$. Similar  results were established in various settings, see~\cite{Gai94, LuoTseng94, Grippo94, ManSol94,bertsekas2000gradient}.

We will provide a short proof for $\lim_{n\to\infty}\nabla f(X_n)=0$ under weak assumptions. Here, we include the case where $\nabla f$ is only Hölder continuous and where $(D_n)_{n \in \N}$ is an $L^p$-martingale difference sequence for a $p\in(1,2]$. Then, we will conclude that we have almost sure convergence of $(X_n)_{n \in \N_0}$ on the event that $(X_n)_{n \in \N_0}$ stays bounded in the case where the set of critical points of $f$ does not contain a continuum, see Theorem~\ref{theo1} below.

In the case where the set of critical points of $f$ contains a continuum of points the situation is more subtle. In that case, Tadic \cite{Tadic15} showed that under a \loja-inequality stochastic gradient descent schemes converges under appropriate additional assumptions. In this article, our considerations are also based on the validity of certain \loja-inequalities. However, here we allow the drift to be more general so that the new convergence theorem is applicable for a bigger class of problems such as   momentum stochastic gradient descent (MSGD). We stress that the proofs developed in this article are significantly different from the ones in \cite{Tadic15}. 
We mention that the asymptotic behaviour of a stochastic gradient descent scheme is tightly related to that of the first order differential equation
\begin{align} \label{eq:ODE}
	\dot x_t =-\nabla f(x_t)
\end{align}
(though many approximation results only hold for a finite time-horizon), see e.g. \cite[Proposition~1]{MHKC20}.
Convergence for the latter differential equation is a non-trivial issue even in the case where the solution stays on a compact set: one can find $C^\infty$-functions~$f$ together with solutions $(x_t)_{t\ge0}$  that stay on compact sets but do not converge, see Example 3 on page 14 of \cite{palis2012geometric}. Counterexamples of this structure have been known for a long time (see e.g.~\cite{curry1944method}) and include the famous Mexican hat function \cite{AMA05}.
To guarantee convergence (at least in the case where the solution stays on a compact set), one needs to impose additional assumptions. An appropriate assumption is the validity of a   \loja-inequality, see Definition~\ref{def:loja} below. This assumption has the appeal that it is satisfied by analytic functions, see~\cite{lojasiewicz1963propriete,lojasiewicz1965ensembles}.

Recently, there is a growing interest in the case where a  \loja-inequality holds 
with exponent $\beta=\frac 12$, often referred to as Polyak--\loja-inequality (or P\CL-inequality). The idea first appeared in 1963 in an article of Polyak \cite{Pol63} where linear convergence in the objective function is shown for non-perturbed gradient descent. In machine learning, this inequality turns out to be quite effective as a tool to weaken strict convexity assumptions and allow multiple global minima, see e.g.~\cite{KNS16, bassily2018exponential, vaswani2019fast,  xie2020linear, gower2021sgd,  wojtowytsch2021stochastic, gess2023convergence} and the references therein.
We stress that still the respective assumptions are significantly stricter than the ones we will impose. In particular, the existence of a global P\CL-inequality implies that the objective function $f$ is quasi-strongly convex, see e.g. \cite{aujol2022convergence, rebjock2023fast}. In contrast, our analysis allows the landscape to contain local minima, maxima or saddle points.

In the main result of this article, we prove convergence of $(X_n)_{n \in \N_0}$ for objective functions that are \loja-functions (i.e. locally satisfy a general \loja-inequality) on the event that the scheme stays bounded,
see Theorem~\ref{theo2} below. In particular, our result applies if the objective function is analytic. This generalises the analysis in~\cite{AMA05} where convergence of gradient descent without stochastic noise is considered. We stress that the limit points may be local minima as well as saddle points. After a preprint version of this article appeared, there has been a lot of interest in the interplay between stochastic optimisation algorithms and loss landscapes satisfying locally Kurdyka-\L ojasiewicz-inequalities (a slight generalisation of the \L ojasiewicz-inequality) and we point to \cite{fatkhullin2022sharp, jentzen2022, chouzenoux2023kurdyka, li2023convergence, milzarek2023convergence} for recent developments in this direction.

As an example we provide a machine learning application: in the case where in a deep learning representation the activation function is analytic (for instance, softplus, hyperbolic tangent, or sigmoid) and the random variables modelling the signal and the response are compactly supported, the respective objective function is analytic, see Theorem~\ref{theo3}. Hence, in that case our results show that a SGD scheme associated with the training of the network converges almost surely on the event of staying bounded.
Concerning the training of neural networks via SGD we also refer the reader to \cite{bach2011non, Bach2013, du19c, karimi2019non, lei2019stochastic, DK20, fehrman2020convergence, cheridito2021nonconvergence, jentzen2023overall}.
Related  objective functions (loss landscapes) are analysed in \cite{nguyen19a, petersen2020topological, qin2020reducing, cooper2021global, dereich2023existence}.

Our analysis is based on non-asymptotic inequalities and if one has additional information about the explicit \loja-type inequalities for the objective function to hold our results entail moment estimates for the distance that can be overcome by the SGD when observing convergence to a particular critical level. 

We proceed with the central definitions and statements.

\begin{defi} \label{def:events}
	\begin{enumerate}\item
		We denote by $\IL$ the event that $(X_n)_{n \in \N_0}$ stays bounded, i.e.,
		\begin{align*}
			\IL=\Bigl\{\limsup_{n\to\infty} |X_n|<\infty \Bigr\}.
		\end{align*}
		\item  We denote by $\mathbb G$ the event 
			\begin{align*}
				\mathbb G = \Bigl\{ \liminf\limits_{n \to \infty} \frac{\langle -\nabla f(X_{n-1}), \Gamma_{n} \rangle }{|\nabla f(X_{n-1})|^2}>0 \text{ and } \liminf\limits_{n \to \infty}\frac{|\nabla f(X_{n-1})|}{|\Gamma_{n}|}>0  \Bigr\},
			\end{align*}
			where we interpret $\sfrac 00$ as $\infty$. 
		\item For $q\ge1$ and a sequence $(\sigma_n)_{n\in\N}$ of strictly positive reals we denote by $\IM^{\sigma,q}$ the event\footnote{ Formally, $\E[D_n|\cF_{n-1}]=0$ is to be understood as condition $\E[D_n^{+}|\cF_{n-1}]=\E[D_n^{-}|\cF_{n-1}]<\infty$, where $D_n^{+}$ and $D_n^{-}$ denote the positive and negative part of $D_{n}$. Note that for general non-negative random variables conditional expectations are always well-defined and, in particular, we do not assume  integrability of $D_{n}$ in this article.}
		\begin{align*}
			\IM^{\sigma,q}=\Bigl\{  &\limsup_{n\to\infty} \sigma_n^{-1}\E[|D_n|^{q}|\cF_{n-1}]^{1/q}<\infty,\\
			&\qquad \text{ and \ } \E[D_n|\cF_{n-1}]=0\text{ for all but finitely many $n$}\Bigr\}.
		\end{align*}
	\end{enumerate}
\end{defi}

\begin{rem} Let us discuss the satisfiability of the latter events.
		First, for continuously differentiable objective functions $f$ with Lipschitz continuous differential that satisfy $\lim_{|x| \to \infty} f(x)=\infty$ the SGD process almost surely stays bounded (see Lemma~D.1 in~\cite{DK20}). In the case that $f$ does not necessarily satisfy the latter property, one can add an $L^2$-regularisation term, i.e., replace $f$ by $\tilde f$ given by
		$
		\tilde f(x) =f(x)+\frac{a}{2} |x|^2
		$
		for a sufficiently large $a >0$ (see Remark~2.8 in~\cite{DK20}).
		Next, when considering stochastic gradient descent, i.e. $(\Gamma_n)_{n \in \N}=(-\nabla f(X_{n-1}))_{n \in \N}$, we clearly have $\P(\IG)=1$. However, the class of optimisation methods satisfying $\P(\IG)=1$ is much broader and include, e.g., line-search methods (see Chapter~3 in~\cite{nocedal1999numerical} and Chapter~4.1 in~\cite{AMA05}), inexact gradient methods (see e.g. \cite{khanh2023new}) and momentum stochastic gradient descent (see Chapter~\ref{sec:MSGD}). The two conditions defining the set $\IG$ guarantee that the size of the drift~$\Gamma_n$ and the negative gradient $-\nabla f(X_{n-1})$ are asymptotically comparable and the angle between the two vectors stays uniformly below $90^\circ$ at late times. See~\cite{chill2009applications} for more examples of gradient-based optimisation dynamics satisfying this angle condition.
		Last, in stochastic optimisation a common assumption on the perturbation is the so-called ABC condition, stating that there exist constants $A,B,C \ge 0$ such that
		\begin{align*}
			\E[|D_n|^2 | \mathcal F_{n-1}] \le 2 A (f(X_{n-1})-f(x^*))+B |\nabla f(X_{n-1})|^2 + C,
		\end{align*}
		where $x^*$ denotes a critical point of $f$ (see e.g.~\cite{gower2021sgd, li2022unified, khaled2022better}). Note that, after establishing boundedness of $(X_n)_{n \in \N}$, the first and second term on the right-hand side in the latter inequality can be bounded by a constant depending on the bounds of the objective function and its differential on the relevant domain. In this work, we allow the size of the stochastic noise to be comparable to a (possibly unbounded) sequence $(\sigma_n)_{n \in \N}$.
	\end{rem}
	
	We state our first main result concerning the convergence of $(f(X_n))_{n \in \N_0}$ and $(\nabla f(X_n))_{n \in \N_0}$. We stress that, comparing to classical results, we weaken the assumptions on the objective function (allowing functions with Hölder continuous differential) and on the stochastic perturbation (allowing martingale differences in $L^{1+\alpha_1}$ for an $\alpha_1 >0$). See~\cite{simsekli2019tail, gurbuzbalaban2021heavy} for heavy-tailed noises in SGD that are in $L^{1+\alpha_1}$ for an $\alpha_1 <1$ but have an infinite second moment. Various sequences $(\sigma_{n})_{n\in\N}$ appear in multilevel stochastic gradient descent algorithms, see~\cite{dereichgronbach2019}.

\begin{theorem}\label{theo1}
	Let  $0<\alpha_1\le \alpha_2\le 1$ and  $(\sigma_n)_{n \in \N}$ be a sequence of strictly positive reals. Suppose that   $\nabla f$ is locally  $\alpha_2$-Hölder continuous, that $\gamma_n \to 0$,
	\begin{align*}
		\sum_{n=1}^\infty (\gamma_n\sigma_n)^{1+\alpha_1}<\infty
	\end{align*}
	and in the case where $\alpha_2<1$ that, additionally,
	\begin{align*}
		\sum_{n=1}^\infty \gamma_n^{\frac{1+\alpha_2}{1-\alpha_2}}<\infty.\end{align*}
	Then, almost surely, on 
	$\IC:=\IL\cap \IG \cap  \IM^{\sigma,1+\alpha_1}$, 
	the limit $\lim_{n\to\infty} f(X_n)$ exists and, if additionally,
	$
	\sum_{n=1}^\infty \gamma_n=\infty,
	$
	one has almost surely, on $\IC$, that
	\begin{align*}
		\lim_{n\to\infty} \nabla f(X_n)=0.
	\end{align*}
	Moreover, in the case where the set of critical points of $f$, $\cC=\{x:\nabla f(x)=0\}$, does not contain a continuum of elements, i.e., there does not exist an injective mapping taking $[0,1]$ to $\cC$, then we have almost sure convergence of $(X_n)_{n \in \N_0}$ on $\IC$.
\end{theorem}

\begin{rem}
	Note that in the case where $\cC$ is countable the set $\cC$ does not contain a continuum of elements.
\end{rem}

Next we prepare the main result of convergence under \loja-type assumptions. 

\begin{defi}\label{def:loja}
	We call a function $f:\R^d\to\R$ \emph{\loja-function}, if $f$ is continuously differentiable with locally Lipschitz continuous differential and if  for every $x\in\cC=\nabla f^{-1}(\{0\})$, the \loja-inequality is true on a neighbourhood $U_x$ of $x$ with parameters $\CL>0$ and $\beta\in[\frac 12,1)$, i.e., for all $y\in U_x$
	\begin{align*}
		|\nabla f(y)|\ge \CL\,|f(y)-f(x)|^\beta.
	\end{align*}
\end{defi}

\begin{rem} 
	Note that for all $x \notin \mathcal C$ and $\beta >0$ there trivially exist a neighbourhood $U_x$ and constant $\CL>0$ such that the \loja-inequality with parameter $\CL$ and $\beta$ hold on $U_x$.
\end{rem}

\begin{theorem}\label{theo2}  Let $f$ be a \loja-function and $q\ge 2$.
	Suppose that for $n \in \N$
	\begin{align*}
		\gamma_n =C_\gamma n^{-\gamma} \text{ \ and \ }\sigma_n=n^{\sigma},
	\end{align*}
	where $C_\gamma>0$, $\gamma \in (\frac 12, 1]$ and $\sigma\in\R$. If 
	\begin{align*}
		\frac 23 (\sigma+1) <\gamma\text{ \ and \ }	  \frac{1}{2\gamma-\sigma-1}<q ,
	\end{align*}
	then, on $\mathbb C:= \IL\cap \IG \cap  \IM^{\sigma,q}$, the process $(X_n)_{n \in \N_0}$ converges, almost surely, to a critical point of $f$.
\end{theorem}

\begin{rem} 
	In the case where $\gamma >1$, $q\in [1,2]$ and $\sigma \in \R$ with $q(\gamma-\sigma)>1$ we clearly have almost sure convergence on $\mathbb C:= \mathbb L \cap \IG \cap  \mathbb M^{\sigma,q}$. However, in that case $\lim_{n \to \infty} X_n$ is not necessary a critical point of $f$. Indeed, let $K$ be a compact set and consider the events
	\begin{align*}
		\IA_n^{N,C,K}=\bigcap_{\ell=N}^n \bigl\{ X_\ell\in K,\,|\Gamma_{\ell+1}|\le  C |\nabla f(X_\ell)| ,\, \E[ |D_{\ell+1}|^{q}\,|\, &\cF_{\ell}]\le C\sigma_{\ell+1}^{q},  \E[D_{\ell+1}\, | \, \cF_n ]=0\bigr\}.
	\end{align*}
	Then, on $\IA_\infty^{N,C,K}= \bigcap_{n\ge N}\IA_n^{N,C,K}$ we have for all $n >N$ 
	\begin{align*}
		|X_n-X_N| \le C \|\nabla f\|_{L^\infty(K)} \sum_{\ell=N+1}^n \gamma_\ell + \Bigl| \sum\limits_{\ell = N+1}^n \gamma_\ell D_\ell \Bigr|.
	\end{align*}
	Moreover, $(M_n)_{n>N}=\big(\sum_{\ell = N+1}^n \1_{\IA_{\ell-1}^{N,C,K}}\gamma_\ell D_\ell\big)_{n>N}$ is a martingale which converges almost surely on $\IA_\infty^{N,C,K}$ due to Lemma~\ref{mart_conv}. The result then follows by taking the countable union over such sets.
\end{rem}

\begin{rem}  \begin{enumerate}
		\item The assumption $q > (2\gamma-\sigma-1)^{-1}$ in Theorem~\ref{theo2} is only a technical condition to control the undershoots of the SGD scheme near a local maximum or saddle point. If, on $\IL \cap \IM^{\sigma,2}$, we have almost surely that $d(X_n, \mathcal M)\to 0$, where $\cM$ denotes the set of all local minima of $f$, then $q=2$ is sufficient for all choices of $\gamma$ and $\sigma$ satisfying $\sfrac 23 (\sigma+1)<\gamma$.
		\item We consider the case, where $(\gamma_n)_{n \in \N}=(C_\gamma n^{-\gamma})_{n \in \N}$ with $C_\gamma >0$ and $\gamma \in (\sfrac 12, 1]$ and $(D_n)_{n \in \N}$ is a sequence of martingale differences with size of order $\cO(1)$.
		In an earlier work, Tadic showed almost sure convergence of stochastic gradient descent \cite{Tadic15} (see also \cite{tadic2009convergence} for an extended version) for analytic objective functions under the assumption that there exists an $r>1$ such that, almost surely,
		\begin{align*} %\label{eq:Tadic}
			\limsup\limits_{n \to \infty} \max_{k \ge n} \Bigl\| \sum_{i=n}^k \gamma_i t_{i}^r D_i \Bigr\|<\infty,
		\end{align*}
		where $t_n = \sum_{i=1}^{n}\gamma_i$.  While this assumption requires $\gamma > \sfrac 34$ our analysis allows us to go as low as $\gamma > \sfrac 23$.
		\item  The SGD scheme shows similar behaviour as the solution to an SDE with a drift term which is comparable in size and direction to $-\nabla f$ plus a diffusion term with vanishing diffusivity. Such SDEs are analysed in~\cite{dereich2022cooling}, where the diffusive term is assumed to have a bracket process whose magnitude is bounded by $t^{-2\rho} \, \dd t$ with $\rho\in (0,\infty)$. If $\rho>1$, one can device a convergence proof for the SDE under analogous assumptions as imposed here. Conversely,   there exists a counterexample with $\rho=1$ for which the SDE diverges and circles around a sphere infinitely often, see~\cite{dereich2022cooling}.
			When translating the SGD problem (with exponent $\gamma<1$) analysed here into the SDE problem the associated~$\rho$ is $\rho=\frac{\gamma-2\sigma}{2-2\gamma}$. Note that $\frac 23 (\sigma+1)$ is the largest  $\gamma$-value for which the respective~$\rho$ is smaller or equal to $1$. This suggests that in Theorem~\ref{theo2} the assumption  $\gamma>\frac 23 (\sigma+1)$ is natural and that the result does in general not hold for smaller~$\gamma$.
\end{enumerate}
\end{rem}

We give two important applications of the main statement above. 
As found by Ruppert \cite{Rup82} and Polyak \cite{Pol90,PJ92}, in many scenarios  the running average of a Robbins--Monro algorithm yields a better performance as the Robbins--Monro algorithm itself. This is even the case where the potential limit points are non-discrete and form a stable manifold, see~\cite{DK20}. Convergence of the algorithm seems to be  (up to very particular examples) a necessity for the Ruppert--Polyak-averaging to have a positive effect. Our research suggests that  Ruppert--Polyak-averaging may also have a positive effect in our setting. At least it entails that the average converges to the same parameter value and thus gives in the limit the same target value (loss) as the Robbins--Monro algorithm itself. 

A promising line of recent research focusses on the implicit bias of gradient descent algorithms, see e.g. \cite{gunasekar2017implicit, soudry2018implicit, arora2019implicit, chizat2019lazy, woodworth2020kernel} and the references mentioned therein. Although the model is often trained using a fixed training set, stochastic gradient descent seems to pick an empirical risk minimiser that generalises well to unseen data. A common state-of-the-art hypothesis in the machine learning community is that implicit bias might be one of the reasons for the outstanding performance of stochastic gradient descent in practice. Our results further motivates the analysis of the implicit bias. In particular, we show that SGD converges to a limit point and one is left to investigate key properties of the limit that may be important to achieve a good generalisation error (such as flatness of the minima, e.g. \cite{wu2018sgd,zhu2019anisotropic, wojtowytsch2021stochasticII}).

The article is arranged as follows. In Section~\ref{sec:MSGD} we apply our results to show convergence of momentum stochastic gradient descent. In Section~\ref{sec2}, we provide the proof of Theorem~\ref{theo1}. Section~\ref{sec3} proceeds with the proof of Theorem~\ref{theo2}. Finally, Section~\ref{sec4} discusses  analytic neural networks and, in particular, analyticity of a particular deep learning network is established in Theorem~\ref{theo3}. Recall that analyticity of the objective function $f$ implies that $f$ is a \L ojasiewicz-function. Therefore, our main results apply for the objective functions that arise from a regression task in supervised learning when using analytic activation functions and an analytic loss function, see Theorem~\ref{theo3}.

\section{Convergence of momentum stochastic gradient descent} \label{sec:MSGD}
Consider the second order system satisfying
\begin{align}\begin{split} \label{eq:MSGDintro}
		X_{n} &= X_{n-1} + \gamma_{n}(V_{n-1}+D_n^{(1)}), \\
		V_{n}&= V_{n-1} - \gamma_{n} (\mu V_{n-1} + \nabla f(X_{n-1}) -  D_{n}^{(2)}),
	\end{split}
\end{align}
for $n \in \N$,
where $\mu>0$, $(D_n^{(1)})_{n \in \N}$ and $(D_n^{(2)})_{n \in \N}$ are $(\mathcal F_{n})_{n \in \N}$-adapted sequences of random variables,  $(\gamma_n)_{n \in \N}$ is a sequence of strictly positive reals and $X_0, V_0$ are $\mathcal F_0$-measurable random variables. Set 
\begin{align*}
	\Gamma_n := \begin{pmatrix}
		V_{n-1} \\ -\mu V_{n-1} - \nabla f(X_{n-1})
	\end{pmatrix}
	\quad \text{ and } \quad D_n:= \begin{pmatrix}
		D_n^{(1)} \\ D_n^{(2)}
	\end{pmatrix}
\end{align*}
and note that (\ref{eq:MSGDintro}) can be written as the first order system (\ref{eq:main}).

\begin{theorem} \label{theoMSGD}
	Let $f$ be a $C^2$-\loja-function such that $\Hess f$ is locally Lipschitz continuous.		
	Under the assumptions of Theorem~\ref{theo2}, we have almost surely on the event $\IC := \IL \cap \IM^{\sigma, q}$ that $V_n \to 0$ and $(X_n)_{n \in \N_0}$ converges to a critical point of $f$.
\end{theorem}

\begin{proof}
	Let $K \subset \R^d$ be a compact set, define $\IL^K = \bigcap_{n \in \N_0}\{X_n \in K\}$ and note that it suffices to show the statement on the event $\IL^K \cap \IM^{\sigma, p}$ since $\IL = \bigcup_{n \in \N} \IL^{K_n}$, where $K_n$ denotes an increasing sequence of compact sets with $\bigcup_{n \in \N} K_n=\R^d$.
	Let $a >0$ and define the energy function
	\begin{align*}
		\cE(x,v) := \frac 12 \|v\|^2+f(x) + a \langle \nabla f(x),v \rangle
	\end{align*}
	so that
	\begin{align*}
		\nabla \cE(x,v) = \begin{pmatrix}
			\nabla f(x) + a \,\Hess f(x) v \\ v+a \,\nabla f(x)
		\end{pmatrix}
		.
	\end{align*}
	Set $C := \sup_{x \in K} \|\Hess f(x)\|$ and note that for all $n \in \N$ we have  $|\Gamma_{n}|^2\le (1+2\mu^2)|V_{n-1}|^2+2|\nabla f(X_{n-1})|^2$ as well as
	\begin{align*}
		|\nabla \cE(X_{n-1},V_{n-1})|^2 &\ge (1-2a+a^2)|\nabla f(X_{n-1})|^2+(1-(C^2+1)a)|V_{n-1}|^2,\\
		|\nabla \cE(X_{n-1},V_{n-1})|^2 &\le (1+2a+a^2)|\nabla f(X_{n-1})|^2+(1+(C^2+1)a+C^2a^2)|V_{n-1}|^2
	\end{align*}
	and
	\begin{align*}
		\langle - \nabla \cE(X_{n-1},V_{n-1}),  \Gamma_{n} \rangle  &\ge (\mu-a C-\sfrac a2 \mu^2) |V_{n-1}|^2+\sfrac a2 |\nabla f(X_{n-1})|^2.
	\end{align*}
	Thus, for sufficiently small $a$ we have 
	\begin{align*}
		\P \Bigl(\IL^K \cap \Bigl\{  \liminf\limits_{n \to \infty} \frac{\langle -\nabla \cE(X_{n-1}, V_{n-1}), \Gamma_{n} \rangle }{|\nabla \cE(X_{n-1}, V_{n-1})|^2}\le 0 \text{ or } \liminf\limits_{n \to \infty}\frac{|\nabla \cE(X_{n-1}, V_{n-1})|}{|\Gamma_{n}|}\le 0  \Bigr\} \Bigr)=0.
	\end{align*}
	Next, we show that on $\IL^K\cap \IM^{\sigma, p}$ we almost surely have $\limsup_{n \to \infty} |V_n|<\infty$. For arbitrary $C_\sigma, C_V>0$ and $N \in \N$ such that $\sup_{n \ge N} \gamma_n < \mu$ consider the adapted sequence of events $(\IA_n)_{n \ge N}$ given by
		\begin{align*}
			\IA_{n} = \{|V_{N}|\le C_V,&\  X_{m} \in K, \  \E[D_{m+1}|\cF_m]=0, \\
			&\E[|D_{m+1}|^2 | \cF_m]\le C_\sigma \sigma_m^2 \text{ for all } m = N, \dots, n \}
		\end{align*}
	for $n \ge N$ and set $C_f := \sup_{x \in K} |\nabla f(x)|$. It suffices to show that the process $(\tilde V_n)_{n \ge N}$ given by $\tilde V_{N} := \1_{\IA_{N}} (|V_{N}| \vee \frac{C_f}{\mu})^2$ and 
	$$
		\tilde V_n := \1_{\IA_{n-1}} \Bigl(|V_n| \vee \frac{C_f}{\mu}\Bigr)^2 - \sum_{i=N+1}^n C_\sigma \gamma_n^2 \sigma_n^2 \quad \text{ for } n > N
	$$
	is almost surely bounded, since $\sum_{n \in \N} \gamma_n^2 \sigma_n^2 < \infty$.
	This immediately follows by Doob's martingale convergence theorem, since $(\tilde V_n)_{n \ge \N}$ is almost surely bounded from below and for all $n > N$
	\begin{align*}
		 \E[\1_{\IA_{n}}|V_n|^2| \, | \cF_{n-1}] &\le \1_{\IA_{n-1}} |(1-\gamma_n \mu)V_{n-1}-\gamma_n \nabla f(X_{n-1})|^2 + C_\sigma \gamma_n^2 \sigma_n^2 \\
		& \le \1_{\IA_{n-1}} |V_{n-1}|^2 + C_\sigma \gamma_n^2 \sigma_n^2
	\end{align*}
	if $|V_{n-1}|>\frac{C_f}{\mu}$, as well as
	$
		\1_{\IA_{n-1}} \E[|V_n|^2| \, | \cF_{n-1}] \le \1_{\IA_{n-1}} \frac{C_f}{\mu} + C\sigma \gamma_n^2 \sigma_n^2
	$ 
	if $|V_{n-1}|\le \frac{C_f}{\mu}$.

	Now, using Theorem~\ref{theo1} we get almost surely on $\IL^K\cap \IM^{\sigma, p}$ that $|\nabla \cE(X_n,V_n)|\to 0$, which for sufficiently small $a$ implies $V_n \to 0$ and $|\nabla f(X_n)| \to 0$. In order to show convergence of $(X_n)_{n \in \N_0}$ we need to show that for sufficiently small $a$ and every $x \in K$ there exists a neighbourhood $U \subset\R^{2d}$ of $(x,0) \in \R^{2d}$ satisfying a \loja-inequality for $\cE$ on $U$. The proof for the latter statement can be found in~\cite[Section~3.2]{chill2009applications}.
\end{proof}

\section{Proof of Theorem~\ref{theo1}}\label{sec2}
In this section, we prove almost sure convergence of $(f(X_n))_{n \in \N_0}$ and $(\nabla f(X_n))_{n \in \N_0}$ under the assumption that the martingale noise is in $L^{1+\alpha_1}$ and $\nabla f$ is $\alpha_2$-Hölder continuous, for $0<\alpha_1\le \alpha_2\le 1$, by analysing the evolution of $(f(X_n))_{n \in \N_0}$ step by step. We use the classical argument, Lemma~\ref{mart_conv}, to show that the influence of the martingale noise is negligible at late times. Then, we prove that Hölder continuity of $\nabla f$ is sufficient for the remainder in the Taylor-approximation to have a negligible effect. We conclude that, for small step-sizes, the system $(f(X_n))_{n \in \N_0}$ has a tendency to decrease and using the boundedness of $f$ on a compact domain we deduce convergence of $(f(X_n))_{n \in \N_0}$. The convergence of $(\nabla f(X_n))_{n \in \N_0}$ and $(X_n)_{n \in \N_0}$ (for functions having isolated critical points) follows from a path-wise analysis using the convergence of $(f(X_n))_{n \in \N_0}$.

\begin{lemma} \label{lem:almostsure}
Let $N\in\N$, $C>0$, $K\subset \R^d$ be a  compact and convex set and denote by $(\IA_n)_{n\ge N}$ a decreasing sequence of events such that for every $n\ge N$, $\IA_n\in\cF_n$ and $\IA_n\subset\{X_n\in K\, , \,  |\nabla f(X_n)|^2\le C \langle \nabla f(X_n), -\Gamma_{n+1} \rangle \, , \,  |\Gamma_{n+1}| \le C |\nabla f(X_n)|\}$.
Further, suppose that $\gamma_n \to 0$, there exists $0<\alpha_1\le \alpha_2\le1$ such that  $\nabla f$ is  $\alpha_2$-Hölder continuous on $K$ and there exists   a sequence $(\sigma_n)_{n \in \N}$ of positive reals such that
\begin{enumerate}
	\item $\E[\1_{\IA_{n-1}} |D_n|^{1+\alpha_1}]^{\frac {1} {1+\alpha_1}}\leq \sigma_n$ for all $n=N+1,\dots$
	\item $\E[\1_{\IA_{n-1}}D_n \,  | \, \cF_{n-1}]=0$ for all $n=N+1, \dots$.
	\item $\sum_{n>N} (\gamma_n\sigma_n)^{1+\alpha_1}<\infty$  and,
	\item in the case where $\alpha_2<1$, $\sum_{n >N} \gamma_n^{\frac{1+\alpha_2}{1-\alpha_2}}<\infty$.
\end{enumerate}
Then on $\IA_\infty= \bigcap_{n \ge N}\IA_n$, almost surely, $(f(X_n))_{n\in\N_0}$ converges and, if additionally,
\begin{enumerate}
	\item[$\mathrm{5.}$] $\sum_{n>N} \gamma_n =\infty$,
\end{enumerate}
then on $\IA_\infty$, almost surely,
$
\lim_{n\to\infty} \nabla f(X_n)=0.
$
\end{lemma}

For the proof of Lemma~\ref{lem:almostsure} we need the following result taken from \cite{luschgy2012martingale}.

\begin{lemma}[{\cite[Corollary~4.19]{luschgy2012martingale}}]\label{mart_conv}
Let $(M_n)_{n\in\N_0}$ be a martingale and, for $n \in \N$, set $\Delta M_n = M_n-M_{n-1}$. If, for a $\beta\in(0,2]$,
\begin{align*}
	\sum_{n\in\N} \E[|\Delta M_n|^\beta ]<\infty,
\end{align*}
then $(M_n)_{n \in \N_0}$ converges almost surely.
\end{lemma}

\begin{proof}[Proof of Lemma~\ref{lem:almostsure}]  In the following $C_1,C_2,\dots$ denote finite constants that only depend on $(\gamma_n)_{n \in \N}$, $\alpha_2$, $\|\nabla f\|_{L^\infty(K)}$ and the Hölder constants of $\nabla f$ on $K$ for the exponents $\alpha_1$ and $\alpha_2$.\smallskip

\noindent	\underline{Step 1:} We prove almost sure convergence of $(f(X_n))_{n \in \N_0}$ on $\IA_\infty$. For $n>N$ we use a Taylor approximation of first order and write
\begin{align*}
	f(X_n)-f(X_{n-1})&= f(X_{n-1}+ \gamma_n(\Gamma_{n}+D_n))-f(X_{n-1})\\
	&=   \underbrace{ \gamma_n \langle f(X_{n-1}), \Gamma_{n}\rangle }_{=:\Delta A_n} +\underbrace {\gamma_n \langle \nabla f(X_{n-1}) ,D_n\rangle}_{=:\Delta M_n}+\Delta R_n,
\end{align*}
where $\Delta R_n$ is the remainder term in the Taylor approximation.	For $n\ge N$ we set  
\begin{align*}
	A_n=\sum_{m=N+1}^n\1_{\IA_{m-1}} \gamma_m \langle \nabla f(X_{n-1}), \Gamma_{n}\rangle \text{ and } M_n =\sum_{m=N+1}^n\1_{\IA_{m-1}} \gamma_m \langle \nabla f(X_{m-1}),D_m\rangle
\end{align*}
and $R_n=\sum_{m=N+1}^n\1_{\IA_{m-1}} \Delta R_m$.	
Note that by the Taylor formula, for all $n>N$, on $\IA_{n}$, we can write $\Delta R_n=\Delta R^{(1)}_n+\Delta R^{(2)}_n$ with 
\begin{align*}
	|\Delta R_n^{(1)}|\le C_1 \gamma_n^{1+\alpha_2} |\Gamma_{n}|^{1+\alpha_2}\text{ \ and \ } |\Delta R_n^{(2)}|\le  C_1 \gamma_n^{1+\alpha_1}|D_n|^{1+\alpha_1},
\end{align*}
where $C_1$ only depends on the Hölder-constants of $\nabla f$ on $K$ with respect to the exponents $\alpha_1$ and~$\alpha_2$. We define $(R_n^{(1)})_{n \ge N}$ and $(R_n^{( 2)})_{n \ge N}$ in analogy to $(R_n)_{n\ge N}$ and note that
\begin{align*}
	\sum_{m = N+1}^\infty \E[\1_{\IA_{m-1} } \gamma_m^{1+\alpha_1}  |D_m|^{1+\alpha_1}]\le \sum_{m =N+1}^\infty (\gamma_m \sigma_m)^{1+\alpha_1} <\infty
\end{align*}
so that $(R_n^{(2)})_{n\ge N}$ converges, almost surely, to a finite value.
Moreover, $(M_n)_{n\ge N}$
is a martingale that converges, almost surely, to a finite value as consequence of Lemma~\ref{mart_conv} and
\begin{align*}
	\sum_{m=N+1}^\infty \E[\1_{\IA_{m-1}}  |\gamma_m \langle  \nabla f(X_{m-1}) ,D_{m}\rangle|^{1+\alpha_1} ] \leq C_2^{1+\alpha_1} \sum_{m=N+1}^\infty   \gamma_m^{1+\alpha_1} \,\E[\1_{\IA_{m-1}} |D_{m}|^{1+\alpha_1} ],
\end{align*}
with $C_2:=\|f\|_{L^\infty(K)}<\infty$. 
Finally, we compare the contribution of the first remainder term $(R^{(1)}_n)_{n\ge N}$ with $(A_n)_{n \ge N}$. 
In the case where $\alpha_2=1$, we have for all $n\ge N$ with $C_1C^{3}\gamma_n\le 1/2$, on $\IA_{n-1}$,
\begin{align}\label{eq7469}
	|\Delta R_n^{(1)}|\le C_1 C^{2} \gamma_n^{2} |\nabla f(X_{n-1})|^{2} \le - C_1 C^{3} \gamma_n^{2} \langle \nabla f(X_{n-1}), \Gamma_{n} \rangle \le - \frac 12 \Delta A_n.
\end{align}
This is the case for all but finitely many $n$'s and we conclude that by monotonicity of $(A_n)_{n \ge N}$ the random variables $(A_n+ R_n^{(1)})_{n\ge N}$ converge almost surely; possibly to minus infinity. We note that  on $\IA_\infty$
\begin{align*}
	f(X_n)=f(X_N) +A_n+R^{(1)}_n +M_n+ R_n^{(2)}
\end{align*}
and by lower boundedness of $f$ on $K$ the latter limit needs to be finite and we proved convergence of $(f(X_n))_{n\in\N_0}$. Note that as consequence of~(\ref{eq7469}), also $(A_n)_{n\ge N}$ converges to a finite value on $\IA_\infty$.

It remains to consider the case where  $\alpha_2<1$. We apply the Hölder inequality with the adjoint exponents $\frac2{1-\alpha_2}$ and $\frac 2{1+\alpha_2}$ and get that
\begin{align*}
	\sum_{m={N+1}}^n &{\gamma_m}^{1+\alpha_2}\1_{\IA_{m-1}} |\nabla f(X_{m-1})|^{1+\alpha_2}\\
	&\le \underbrace{ \Bigl(\sum_{m={N+1}}^\infty {\gamma_m}^{\frac{1+\alpha_2}{1-\alpha_2}}\Bigr)^{\frac {1-\alpha_2}2}}_{<\infty}\,\Bigl(\sum_{m=N+1}^n \1_{\IA_{m-1}}\gamma_m|\nabla f(X_{m-1})|^2\Bigr)^{\frac{1+\alpha_2}{2}}.
\end{align*}
Since $\frac {1+\alpha_2}2$ is less than one we get with monotonicity of $(A_n)_{n \ge N}$ that $(A_n+R^{(1)}_n)_{n\ge N}$ converges almost surely; again possibly to an infinite value. We proceed as above to conclude that the latter limit is finite and that the sequences $(f(X_n))_{n\ge N}$ and $(A_n)_{n\ge N}$ are almost surely convergent on $\IA_\infty$ with finite limits. 
\smallskip

\noindent
\underline{Step 2:} We prove almost sure convergence of $(\nabla f(X_n))_{n \in \N_0}$ to $0$ on $\IA_\infty$ in the case that $\sum \gamma_n =\infty.$
First note that in analogy to above  the process $(\tilde M_n)_{n\ge N}$ given by
\begin{align*}
	\tilde M_n = \sum_{m=N+1}^n \1_{\IA_{m-1}} \gamma_m D_m
\end{align*}
is  a martingale that converges, almost surely, as consequence of Lemma~\ref{mart_conv}.
Consequently, we get with the first step that
\begin{align*}
	\Omega_0:=\Bigl\{\inf_{n\ge N} A_n > - \infty\Bigr\}\cap \Bigl\{\lim_{n\to\infty} \tilde M_n \text{ exists  in }\R^d\Bigr\}
\end{align*}
is an almost sure event.
Suppose now that there exists $\omega\in \Omega_0\cap \IA_\infty$ for which $(\nabla f(X_n(\omega)))_{n \in \N_0}$ does not converge to zero. Then there exist $\delta>0$ and a strictly increasing sequence $(n_k)_{k\in\N}$ with 
$
| \nabla f(X_{n_k}(\omega))|>\delta
$
for all $k\in \N$. Note that $|\nabla f|$ is uniformly bounded over the set $K$ by $\kappa$ and choose $\eps>0$ such that 
\begin{align*}
	\sup_{x,y\in K: |x-y|\leq \eps} |\nabla f(x)-\nabla f(y)|\leq \delta/2.
\end{align*}
We let $(m_k)_{k\in\N}$ such that
\begin{align*}
	\sum_{\ell=n_k+1}^{m_k} \gamma_\ell \leq \frac{\eps } {2C \kappa} < \sum_{l=n_k+1}^{m_k+1} \gamma_\ell.
\end{align*}
By thinning the original sequence $(n_k)_{k \in \N}$ we can ensure that $([n_k+1,m_k]: k\in\N)$ are disjoint intervals. Moreover, by discarding the first terms from the sequence $(n_k)_{k \in \N}$ we can ensure that for all $k\in\N$
\begin{align*}
	\sup_{m\geq n_k} |\tilde M_m(\omega)-\tilde M_{n_k}(\omega)|< \eps/2.
\end{align*}
Now note that for $\ell=n_k,\dots, m_k$,
\begin{align*}
	|X_\ell(\omega)-X_{n_k}(\omega)|\leq   C \kappa  \sum _{i=n_k+1}^{m_k} \gamma_i +|\tilde M_\ell(\omega)-\tilde M_{n_k}(\omega)| \leq \eps.
\end{align*}
Consequently, for these $\ell$
\begin{align*}
	|\nabla f(X_\ell(\omega))| \geq |\nabla f(X_{n_k}(\omega))|- \delta/2 \geq \delta/2,
\end{align*}
and hence 
\begin{align*}
	\sum_{\ell=n_k+1}^{m_k} \gamma_{\ell} |\nabla f(X_{\ell-1}(\omega))|^2 \geq \bigl(\sfrac \delta2\bigr)^2 \sum_{\ell=n_k+1}^{m_k} \gamma_\ell \to \bigl(\sfrac \delta2\bigr)^2\sfrac\eps{2\kappa}.
\end{align*}
Since the intervals $([n_k+1,m_k]: k\in\N)$ are pairwise disjoint and since $\omega\in\IA_\infty$  we get that 
$
\sum_{k=N+1}^{\infty} \gamma_k \1_{\IA_{k-1}}(\omega) \, |\nabla f(X_{k-1}(\omega))|^2
=\infty
$
which contradicts $\omega\in \Omega_0$. Thus we proved that $(\nabla f(X_n))_{n \in \N_0}$ converges to zero on $\Omega_0\cap\IA_\infty$.
\end{proof}

\begin{proof}[Proof of Theorem~\ref{theo1}]
Let $N\in\N,\, C>0$ and $K\subset \R^d$ compact and consider the sets $(\IA_{n}^{N,C, K}:n\ge N)$ given by 
\begin{align*}
	\IA_n^{N,C,K}=\bigcap_{\ell=N}^n \bigl\{  X_\ell\in K,\,  & |\nabla f(X_\ell)|^2\le C \langle \nabla f(X_\ell), -\Gamma_{\ell+1} \rangle, \, |\Gamma_{\ell+1}| \le C |\nabla f(X_\ell)|, \\
	& \E[ |D_{\ell+1}|^{1+\alpha_1}\,|\, \cF_{\ell}]\le C\sigma_{\ell+1}^{1+\alpha_1}, \, \E[D_{\ell+1}\, | \, \cF_n ]=0\bigr\}.
\end{align*}
The sets satisfy the assumptions of Lemma~\ref{lem:almostsure} and we thus have on $\IA_\infty^{N,C,K} =\cap_{n\ge N}  \IA_n^{N,C,K}$ almost sure convergence of $(f(X_n))_{n\in\N}$ and, in the case $\sum \gamma_n =\infty$,
$
\lim_{n\to\infty} \nabla f(X_n)=0.
$
Now the statement follows, since for an increasing sequence of compact sets $(K_N: N\in\N)$ with $\bigcup K_N=\R^d$ we have
\begin{align}\label{eq9359}
	\mathbb C= \bigcup_{N\in\N} \IA_\infty^{N, N,K_N}.
\end{align}

It remains to prove almost sure convergence of $(X_n)_{n \in \N_0}$ on $\IC$ in the case where the set $\cC=\{x:\nabla f(x)=0\}$,  
does not contain a continuum of points.
By (\ref{eq9359}), it suffices to show almost sure convergence on $\IA_\infty^{N, N,K_N}$ for arbitrary $N\in\N$.
First, we verify that, almost surely, $\lim_{n\to \infty} (X_{n}-X_{n-1})= 0$. Indeed, as shown in the proof of Lemma~\ref{lem:almostsure} one has on $\IA_\infty^{N, N,K_N}$ that $(\sum_{\ell=N+1}^n \gamma_\ell D_\ell)_{n\ge N}$ converges, almost surely, so that $\gamma_n D_n\to 0$ on $\IA_\infty^{N, N,K_N}$. Moreover, $\gamma_n |\Gamma_{n+1}| \le \gamma_n C |\nabla f(X_n)| \to 0$ on $\IA_\infty^{N, N,K_N}$ so that $(X_n-X_{n-1})_{n\in\N}$ almost surely converges to zero on $\IA_\infty^{N,N,K_N}$. 
Now note that
\begin{align*}
	&\P(\{(X_n)_{n \in \N_0} \text{ diverges}\}\cap \IA_\infty^{N, N,K_N})\\
	&\le \sum_{i=1}^d \P((X^{(i)}_n)_{n \in \N_0} \text{ diverges}, \nabla f(X_n)\to 0, X_n-X_{n-1}\to 0, X_m\in K_N \,  \forall \, m\ge N),
\end{align*}
where $X_n^{(i)}$ denotes the $i$-th coordinate of the random vector $X_n \in \R^d$.

If $\P(\{(X_n)_{n \in \N_0} \text{ diverges}\}\cap \IA_\infty^{N, N,K_N})$ were strictly positive, one of the summands on the right-hand side would be strictly positive. Suppose that this were true for the $i$-th summand. In particular, for this $i$ the respective event does not equal the empty set.
We choose 
\begin{align*}
	\omega\in\{(X^{(i)}_n)_{n \in \N_0} \text{ diverges}, \nabla f(X_n)\to 0, X_n-X_{n-1}\to 0, X_m\in K_N \,  \forall \, m\ge N\}.
\end{align*}
Thus, 
\begin{align*}
	a:=\liminf _{n\to\infty}X_n^{(i)}(\omega)\neq  \limsup_{n\to\infty} X_n^{(i)}(\omega)=:b.
\end{align*}
By $X_n(\omega) -X_{n-1}(\omega)\to 0$ we also get $\mathrm{acc}((X^{(i)}_n(\omega))_{n \in \N_0})=[a,b]$ and the principle of nested  intervals produces for every $u\in [a,b]$ an accumulation point of $(X_n(\omega))_{n \in \N_0}$ with $i$-th coordinate equal to $u$. Hence, we defined an injective mapping that maps  each $u\in [a,b]$ to a different accumulation point $\eta(u)$ of $(X_n(\omega))_{n \in \N_0}$. By choice of $\omega$, $\eta(u)$ is a critical point of $f$ and we thus constructed a continuum of critical points.
\end{proof}

\section{Proof of Theorem~\ref{theo2}}\label{sec3}
The proof of Theorem~\ref{theo2} is arranged as follows. In Section~\ref{sec31}, we provide moment estimates for the objective function value seen by the SGD where we fix a particular critical level and ``take out'' realisations that undershoot the critical level by a certain amount (\emph{lower dropdown}). In Section~\ref{sec32}, we provide estimates for   the drift term $(\sum_{\ell=1}^n \gamma _\ell \Gamma_\ell)_{n\in\N}$ in the setting of Section~\ref{sec31} on the event that no dropdown occurs.
In Section~\ref{sec33}, we provide estimates that later allow us to show that when a dropdown occurs for a particular critical level it is very likely that the limit of the  target value of the SGD lies strictly below that critical level. In Section~\ref{sec34}, we provide a technical lemma which shows applicability of the previous results in the context of Theorem~\ref{theo2}. 
In Section~\ref{sec35}, we provide properties of \loja-functions and in Section~\ref{sec36} we combine the results to achieve the proof of Theorem~\ref{theo2}.

\subsection{An estimate for the target value in the case without lower dropdown}\label{sec31}

Technically, we will analyse  the evolution of $(f(X_n))_{n \in \N_0}$ step by step and will thereby restrict attention to certain ``nice'' events. For this we use the concept of \emph{compatible} events.

\begin{defi}Let $N\in\N$, $\delta, C>0$ $U\subset \R^d$, $(\sigma_n)_{n>N}$ a sequence of positive reals and $q\ge2$. We call events $(\IA_n)_{n\ge N}$ \emph{$(U,\delta,C, (\sigma_n),q)$-compatible}, if 
$(\IA_n)_{n \ge  N}$ is a sequence of decreasing events such that for all $n\ge N$
\begin{enumerate}
	\item[(i)]  $\IA_n\in \cF_n$ (adaptivity)
	\item[(ii)]  $\IA_n\subset \{X_n\in U  , \,  |\nabla f(X_n)|^2 \le C \langle \nabla f(X_n), -\Gamma_{n+1} \rangle  , \,   |\Gamma_{n+1}|\le C |\nabla f(X_n)| , \, \delta \ge   \gamma_{n+1} |\Gamma_{n+1}| \}$ (boundedness and drift condition)
	\item[(iii)] $\E[\1_{\IA_n} D_{n+1}|\cF_n]=0$ (martingale condition)
	\item[(iv)] $\E[\1_{\IA_n} |D_{n+1}|^q|\cF_n]\le \sigma_{n+1}^q$  (moment condition)
	\item[(v)] $\IA_{n+1}\subset \{ \gamma_{n+1} |D_{n+1}|\le \delta\}$ (excess condition).  
\end{enumerate}
\end{defi}

The proof of the main result is based on several preliminary results. At first we provide an upper bound for the $f$-value of $(X_n)_{n \in \N_0}$ when approaching a critical level with a lower dropdown, meaning that we only consider realisations of the SGD that do not undershoot the critical level significantly. 
We will use the ``natural'' time scale associated with a stochastic approximation scheme which associates the $n$-th iterate $X_n$ with the time
$
t_n=\sum_{\ell=1}^n \gamma_\ell
$
and see that the $f$-value is bounded from above by a term arising from the stochastic noise and the function $(\Phi_t^R)_{t \ge 0}$, the worst-case bound for the $f$-value of the solution to the ODE (\ref{eq:ODE}) in a region satisfying a \loja-inequality with exponent $\beta \in (\frac 12, 1)$ and constant $\sqrt{\eta}>0$ starting in $R$, i.e. the solution to the ODE
\begin{align*}
	\dot\Phi^{(R)}_{t}=-\eta  \bigl( \Phi^{(R)}_{t} \bigr)^{2\beta}, \quad \Phi_0^{(R)}=R,
\end{align*}
see e.g.~\cite{haraux2012some}.
For a subset $U \subset \R^d$ we denote $\|\nabla f\|_{\mathrm{Lip(U)}}= \sup\limits_{x,y\in U, x \neq y}\frac{|\nabla f(x)-\nabla f(y)|}{|x-y|}$.

\begin{prop}\label{le:loja}Suppose that the following assumptions are satisfied:\smallskip

\noindent{\bf \loja-assumption.}
Let $U$ be an open and bounded  set containing a critical point $x_0$ and suppose that the \loja-inequality holds on $U$ with parameters $\CL>0$ and $\beta\in(\frac 12,1)$  meaning that for all $x\in U$
\begin{align*}
	|\nabla f(x)|\ge \CL |f(x)-f(x_0)|^\beta.
\end{align*} \smallskip

\noindent{\bf Compatible events.} Let $N\in\N_0$, $\delta, C>0$, $q\ge 2$ and let $(\sigma_n)_{n> N}$ and $(w_n)_{n\ge N}$ be  sequences of positive reals. We denote by $(\IA_n)_{n\ge N}$  a  $(U,\delta,C, (\sigma_n),q)$-compatible sequence of events. We let $\IB_N=\IA_N\cap\{ f(X_N)-f(x_0)\ge -w_N\}$ and for $n>N$
\begin{align*}
	\IB_n=\IA_{n-1}\cap\{\gamma_n|D_n|<\delta, f(X_n)-f(x_0)\ge - w_n\},
\end{align*}
and suppose that   $\IB_n\supset \IA_{n}$ for $n\ge N$. \smallskip

\noindent {\bf Main assumptions for asymptotic error term.} Let  $(v_n)_{n\ge N}$ and $(\gamma_n)_{n \in \N}$
be decreasing sequences of positive reals that satisfy for a constant $\kappa>0$ for  all $n>N$
\begin{align}\label{eq235} \frac {v_{n-1}}{v_{n}}-1\le \kappa  \gamma_n v_{n}^{2\beta-1}\end{align}
and
\begin{align}\begin{split}\label{eq835}
		\bigl(\delta^{-1}& \|\nabla f\|_{L^\infty(U)}+2 \|\nabla f\|_{\mathrm{Lip}(U^{2\delta})} \bigr)(\gamma_n\sigma_n)^2\\
		&+\bigl( \|\nabla f\|^q_{L^\infty(U)}+(2\|\nabla f\|_{\mathrm{Lip}(U^{2\delta})}\bigr)^{q/2}) w_n^{-(q-1)}(\gamma_n\sigma_n)^q\le \gamma_n v_n^{2\beta},
\end{split}	\end{align}
where $U^{2\delta}=\{x\in \R^d: d(x,U)<2\delta\}$.\smallskip

\noindent	{\bf Additional technical assumptions.}
We assume that for a constant $\delta'\in(0,1)$ we have $2C^3\|\nabla f\|_{\mathrm{Lip}(U^{2\delta})} \gamma_n\le \delta'$ for all $n>N$ as well as
\begin{align} \label{eq:332211}
	\frac 12 w_{n-1} \le w_n\le v_n\wedge 1  \text{ \, and \, } ((C \vee C^{-1})\gamma_n+2\|\nabla f\|_{\mathrm{Lip}( U^{2\delta})} C^2\gamma_n^2)  \|\nabla f\|_{L^\infty(U)} ^2\le w_n.
\end{align}

\noindent{\bf Assumptions for the associated master equation.}
Let $y_0>0, \eta>0$ such that for
\begin{align*}
	g: \R\to \R, \,y\mapsto \kappa y-(1-\delta')C^{-1} \CL^2|y|^{2\beta}+1
\end{align*}
one has for all $y\ge0$
\begin{align*}
	\Bigl(\sup_{n>N}\frac{v_n^{2\beta}}{v_{n-1}^{2\beta}}\Bigr) g(y+y_0)\le -\eta |y|^{2\beta}.
\end{align*}

\noindent{\bf Associated differential equation.} For $R\ge 0$ and $t\ge0$ we let %and $n\in\N_0$ we let  	
\begin{align*}
	\Phi^{(R)}_t=R\Bigl(\big((2\beta-1)\eta R^{2\beta-1}\big)t+1\Bigr)^{-1/(2\beta-1)}. \end{align*}

\noindent	{\bf Result.} Suppose that the above assumptions hold and that for a $R>0$ with  $R+(y_0+8)v_N\le ((1-\delta')2\beta C^{-1}\CL^2 \gamma_{N+1})^{-1/(2\beta-1)}$ one has $\E[\1_{\IB_N} (f(X_N)-f(x_0)) ]\le R$. Then one has, for all $n\ge N$,
\begin{align}\label{eq:result}
	\E\bigl[\1_{\IB_n} (f(X_n)-f(x_0)) \bigr]\le \Phi^{(R)}_{t_n-t_N}+ (y_0+8)v_n.
\end{align}
\end{prop}

\begin{proof}We assume without loss of generality that $f(x_0)=0$ and we write, for $n >N$, 
\begin{align} \label{eq:92736453}
	f(X_n)=f(X_{n-1}) + \gamma_n\langle \nabla f(X_{n-1}),\Gamma_{n}+D_n\rangle+R_n.
\end{align}
Taylor's formula gives that in the case where the segment joining $X_{n-1}$ and $X_n$ lies in $\tilde U:=U^{2\delta}$, the remainder satisfies
\begin{align} \label{eq:445344}
	|R_n|\le 2\|\nabla f\|_{\mathrm{Lip}(\tilde U)} \gamma_n^2(|\Gamma_{n}|^2+|D_n|^2).
\end{align}
We note that this is indeed the case when $\IB_n$ enters. With 
$
\xi_n:=\E[\1_{\IB_n}f(X_n)]
$
one has
$
\xi_n= \mathrm{I}+\mathrm{II},
$
where
\begin{align*}
	\mathrm{I}= -\E[\1_{\IA_{n-1}\cap \{\gamma_n |D_n|<\delta\}\cap\{ f(X_n)<-w_n\}} f(X_n)]
\end{align*}
and
\begin{align*}
	\mathrm{II}=\E[\1_{\IA_{n-1}\cap \{\gamma_n |D_n|<\delta\}} f(X_n)].
\end{align*}
We analyse the terms separately.\medskip

\noindent\emph{Analysis of I}.	
On $\IA_{n-1}\cap\{\gamma_n|D_n|<\delta\}$, one has that
\begin{align*}
	|f(X_n)-f(X_{n-1})|\le & (C \gamma_n +2 \|\nabla f\|_{\mathrm{Lip}(\tilde U)} C^2 \gamma_n^2)  \|\nabla f\|_{L^\infty(U)} ^2\\
	&+ \|\nabla f\|_{L^\infty(U)} \gamma_n|D_n|+2\|\nabla f\|_{\mathrm{Lip}(\tilde U)} \gamma_n^2 |D_n|^2.
\end{align*}
Using that by assumption $(C\gamma_n+2\|\nabla f\|_{\mathrm{Lip}(\tilde U)} C^2 \gamma_n^2)  \|\nabla f\|_{L^\infty(U)} ^2\le w_n$, we get that on  $\IA_{n-1}\cap\{\gamma_n|D_n|<\delta\}$,
\begin{align*}
	& \1_{ \{|f(X_n)-f(X_{n-1})|\ge 3w_n\}}(|f(X_n)-f(X_{n-1})|- 3w_n)\\
	& \le  \1_{\{\|\nabla f\|_{L^\infty (U)} \gamma_n|D_n|\ge w_n\}}\|\nabla f\|_{L^\infty( U)} \gamma_n|D_n|    +\1_{\{2\|\nabla f\|_{\mathrm{Lip}(\tilde U)} \gamma_n^2 |D_n|^2\ge w_n\}} 2\|\nabla f\|_{\mathrm{Lip}(\tilde U)} \gamma_n^2 |D_n|^2.
\end{align*}
Now,
\begin{align*}
	\E[\1_{\IA_{n-1}}\1_{\{\|\nabla f\|_{L^\infty(U)} \gamma_n|D_n|\ge w_n\}}\|\nabla f\|_{L^\infty(U)} \gamma_n|D_n|]\le   \|\nabla f\|_{L^\infty(U)}^q w_n^{-(q-1)}(\gamma_n \sigma_n)^q
\end{align*}
and, analogously,
\begin{align*}
	\E[\1_{\IA_{n-1}}&\1_{\{2\|\nabla f\|_{\mathrm{Lip}(\tilde U)} \gamma_n^2 |D_n|^2\ge w_n\}} 2\|\nabla f\|_{\mathrm{Lip}(\tilde U)} \gamma_n^2 |D_n|^2]\\
	&\le   (2\|\nabla f\|_{\mathrm{Lip}(\tilde U)})^{q/2}w_n^{-(q/2-1)} (\gamma_n \sigma_n)^q.
\end{align*}
Setting $C_1= \|\nabla f\|^q_{L^\infty(U)}+(2\|\nabla f\|_{\mathrm{Lip}(\tilde U)})^{q/2}$ and using $w_n\le 1$ we get that
\begin{align*}
	&\E[\1_{\IA_{n-1}\cap\{\gamma_n |D_n|<\delta,  |f(X_n)-f(X_{n-1})|\ge 3w_n\}} (|f(X_n)-f(X_{n-1})|-3w_n)]\\
	&\le  C_1 w_n^{-(q-1)}(\gamma_n\sigma_n)^q.
\end{align*}
Recall that $\IA_{n-1}\subset \IB_{n-1}$ so that, on $\IA_{n-1}$, $f(X_{n-1})\ge -w_{n-1}$. Thus, using $w_{n-1}\le 2 w_n$ we get
	\begin{align}\begin{split} \label{eq957638}
			\mathrm{I}&\le \E[\1_{\IA_{n-1}\cap \{\gamma_n |D_n|<\delta\}\cap\{ f(X_n)<-w_n\}} (|f(X_n)-f(X_{n-1})|+w_{n-1})] \\
			&\le  { 5w_n} \,\P(\IA_{n-1}\cap \{\gamma_n |D_n|<\delta\}\cap\{ f(X_n)<-w_n\})+C_1 w_n^{-(p-1)}(\gamma_n\sigma_n)^p. 
		\end{split}
	\end{align}
\medskip

\noindent \emph{Analysis of II.} 	
Again, one has
\begin{align*}
	\mathrm{II}
	&=\E\bigl[\1_{\IA_{n-1}\cap \{\gamma_n |D_n|<\delta\}} \bigl(f(X_{n-1})+ \gamma_n \langle \nabla f(X_{n-1}), \Gamma_{n}+D_n\rangle +R_n\bigr)\bigr],
\end{align*}
where the remainder satisfies inequality (\ref{eq:445344})
since on $\IA_{n-1}\cap \{\gamma_n |D_n|<\delta\}$ the segment connecting $X_{n-1}$ and $X_n$ lies in $\tilde U$. By assumption, $2C^3 \|\nabla f\|_{\mathrm{Lip}(\tilde U)}\gamma_n \le \delta'$ so that
\begin{align*}
	\mathrm{II}\le \E\bigl[\1_{\IA_{n-1}\cap \{\gamma_n |D_n|<\delta\}} \bigl(&f(X_{n-1})-(1-\delta') C^{-1} \gamma_n |\nabla f(X_{n-1})|^2 \\
	&+\gamma_n\langle \nabla f(X_{n-1}),D_n\rangle +2 \|\nabla f\|_{\mathrm{Lip}(\tilde U)} \gamma_n^2|D_n|^2\bigr)\bigr].
\end{align*}
We note that
\begin{align*}
	\E\bigl[\1_{\IA_{n-1}\cap \{\gamma_n |D_n|<\delta\}}\langle \nabla f(X_{n-1}),D_n\rangle\bigr]&=-\E\bigl[\1_{\IA_{n-1}\cap \{\gamma_n |D_n|\ge\delta\}}\langle \nabla f(X_{n-1}),D_n\rangle\bigr]\\
	&\le  \delta^{-1} \|\nabla f\|_{L^\infty(U)} \gamma_n \sigma_{n}^2.
\end{align*}
Moreover, using that $C^{-1} \gamma_n \|\nabla f\|_{L^\infty(U)}^2\le w_n$ and $w_{n-1}\le 2 w_n$ we conclude that
\begin{align*}
	&-\E\bigl[\1_{\IB_{n-1}\cap (\IA_{n-1}^c\cup\{\gamma_n |D_n|\ge \delta\})} \bigl(f(X_{n-1})-(1-\delta') C^{-1} \gamma_n |\nabla f(X_{n-1})|^2\bigr)\bigr]\\
	&\ \le  \P(\IB_{n-1}\cap (\IA_{n-1}^c\cup\{\gamma_n |D_n|\ge \delta\}) ( 2w_n+C^{-1} \gamma_n \|\nabla f\|_{L^\infty(U)} ^2)\\
	&\ \le 3\P(\IB_{n-1}\backslash \IB_n) w_n.
\end{align*}
Consequently,
\begin{align*}
	& \E\bigl[\1_{\IA_{n-1}\cap \{\gamma_n |D_n|<\delta\}} \bigl(f(X_{n-1})-(1-\delta') C^{-1} \gamma_n |\nabla f(X_{n-1})|^2\bigr)\bigr]\\
	&=\E\bigl[\1_{\IB_{n-1}} \bigl(f(X_{n-1})-(1-\delta')C^{-1} \gamma_n |\nabla f(X_{n-1})|^2\bigr)\bigr]\\
	&\qquad -\E\bigl[\1_{\IB_{n-1}\cap (\IA_{n-1}^c\cup\{\gamma_n |D_n|\ge \delta\})} \bigl(f(X_{n-1})-(1-\delta')C^{-1} \gamma_n |\nabla f(X_{n-1})|^2\bigr)\bigr]\\
	&\le  \E\bigl[\1_{\IB_{n-1}} \bigl(f(X_{n-1})-(1-\delta')C^{-1} \gamma_n |\nabla f(X_{n-1})|^2\bigr)\bigr]
	+3\P(\IB_{n-1}\backslash \IB_n) w_n.
\end{align*}
We conclude that altogether 
\begin{align*}
	\mathrm{II}& \le \E\bigl[\1_{\IB_{n-1}} \bigl(f(X_{n-1})-(1-\delta') C^{-1} \gamma_n |\nabla f(X_{n-1})|^2\bigr)\bigr]+3\P(\IB_{n-1}\backslash \IB_n) w_n\\
	&\hspace{2.8cm}  +(\delta^{-1} \|\nabla f\|_{L^\infty(U)} +2 \|\nabla f\|_{\mathrm{Lip}(\tilde U)} )(\gamma_n\sigma_n)^2.
\end{align*}
\medskip

\noindent \emph{Estimating $\xi_n$ in terms of a difference equation.}	
Combining the estimates for I and II, see~(\ref{eq957638}) above, we get that
\begin{align}\begin{split}\label{eq985761}
		\xi_n&\le   \E\bigl[\1_{\IB_{n-1}} \bigl(f(X_{n-1})-(1-\delta') C^{-1} \gamma_n |\nabla f(X_{n-1})|^2\bigr)\bigr]+8w_n \P(\IB_{n-1}\backslash \IB_{n})  \\
		&\qquad +\underbrace{(\delta^{-1} \|\nabla f\|_{L^\infty(U)}+2 \|\nabla f\|_{\mathrm{Lip}(\tilde U)} )(\gamma_n\sigma_n)^2+C_1 w_n^{-(q-1)}(\gamma_n\sigma_n)^q}_{\le \gamma_n v_n^{2\beta}}.
\end{split}\end{align}
Using the \loja-inequality together with the convexity of $x\mapsto |x|^{2\beta}$ we get that
\begin{align*}
	\xi_n&\le  \xi_{n-1}- (1-\delta')C^{-1} \CL^2 \gamma_n |\xi_{n-1}|^{2\beta}+ 8w_n \P(\IB_{n-1}\backslash \IB_{n})  +\gamma_n v_n^{2\beta}.
\end{align*}
By assumption, one has for $n>N$ that
$
\frac {v_{n-1}}{v_n}-1\le \kappa \gamma_n v_n^{2\beta-1}
$
and using that $(v_n)_{n \ge N}$ is monotonically decreasing and $w_n \le v_n$ we get for  $\zeta_n:=\xi_n/v_n$
and $	g(y)=\kappa y- (1-\delta')C^{-1} \CL^2|y|^{2\beta}+1$
that
\begin{align}\label{eq:dif} \zeta_n\le \zeta_{n-1}+\gamma_nv_n^{2\beta-1} g(\zeta_{n-1})+8 \P(\IB_{n-1}\backslash \IB_n).
\end{align}
\medskip

\noindent \emph{A comparison argument for the difference equation~(\ref{eq:dif}).}
For $s>0$, the mapping 
\begin{align*}
	T_s:\R\to \R, \, y\mapsto y+s\, g(y)
\end{align*}
is monotonically increasing on $(-\infty,((1-\delta')2\beta C^{-1} \CL^2 s)^{-1/(2\beta-1)}]$.
We let 
\begin{align*}
	(\bar \zeta_n)_{n\ge N}=(\Phi_{t_n-t_N}^{(R)}/v_n+y_0)_{n\ge N}
\end{align*}
and note that, for $n>N$,
\begin{align*}
	(\bar \zeta_{n-1}+8) ((1-\delta') & 2\beta C^{-1}\CL^2 \gamma_n v_n^{2\beta-1})^{1/(2\beta-1)}\\
	&\le (v_{n-1}\bar \zeta_{n-1}+8v_{n-1}) ((1-\delta') 2\beta C^{-1}\CL^2 \gamma_n )^{1/(2\beta-1)}\\
	&\le (v_{N}\bar \zeta_{N}+8v_{N}) ((1-\delta') 2\beta C^{-1}\CL^2 \gamma_{N+1} )^{1/(2\beta-1)}\le1,
\end{align*}
so that $T_{\gamma_n v_n^{2\beta-1}}$ is monotonically increasing on $(-\infty, \bar \zeta_{n-1}+8]$. We prove by induction that for all $n\ge N$,
$
\zeta_n\le \bar \zeta_n+8\,\P(\IB_N\backslash \IB_n).
$
By assumption, $\zeta_N=\xi_N/v_N\le R/v_N\le \bar \zeta_N$. Moreover, supposing that the statement is true for $n-1\ge N$, we conclude that $\zeta_{n-1}\le \bar \zeta_{n-1}+8$ so that we get with~(\ref{eq:dif}) and  the monotonicity of $T_{\gamma_n v_n^{2\beta-1}}$ on $(-\infty, \bar \zeta_{n-1}+8]$   that
\begin{align}\begin{split}
		\zeta_n&\le T_{\gamma_n v_n^{2\beta-1}} (\bar \zeta_{n-1}+8\,\P(\IB_N\backslash \IB_{n-1})) +8\P(\IB_{n-1}\backslash\IB_{n}) \\
		&= \bar \zeta_{n-1} + \gamma_n v_n^{2\beta-1} g(\bar \zeta_{n-1}+8\,\P(\IB_N\backslash \IB_{n-1}) ) +8\P(\IB_{N}\backslash\IB_{n}) \\
		&\le \frac{\Phi_{t_{n-1}-t_N}^{(R)}}{v_{n-1}} -\eta \gamma_n v_n^{2\beta-1} \Bigl( \frac{\Phi^{(R)}_{t_{n-1}-t_N}}{v_{n}}\Bigr)^{2\beta} +y_0+8\P(\IB_{N}\backslash\IB_{n})\\
		&\le \frac 1{v_n} \bigl(\Phi_{t_{n-1}-t_N}^{(R)}  -\eta \gamma_n \bigl( \Phi^{(R)}_{t_{n-1}-t_N} \bigr)^{2\beta}\bigr) +y_0+8\P(\IB_{N}\backslash\IB_{n}).\label{eq98357}\end{split}
\end{align}
Noting that $\dot\Phi^{(R)}_{t_{n-1}-t_N}=-\eta  \bigl( \Phi^{(R)}_{t_{n-1}-t_N} \bigr)^{2\beta}$ we get that 
\begin{align*}
	-\eta\gamma_n\bigl( \Phi^{(R)}_{t_{n-1}-t_N} \bigr)^{2\beta}\le \int_{t_{n-1}-t_N}^{t_n-t_N}\dot\Phi^{(R)}_u\, du= \Phi_{t_{n}-t_N}^{(R)}-\Phi_{t_{n-1}-t_N}^{(R)}.
\end{align*}
With~(\ref{eq98357}) it follows that $\zeta_n\le \bar \zeta_n+8\P(\IB_N\backslash \IB_n)$ which finishes the proof.
\end{proof}

\begin{prop}\label{prop583}Let  $\beta\in(\frac 12,1)$ and $(\gamma_n)_{n \in \N}$ and $(\sigma_n)_{n \in \N}$ be sequences of strictly positive reals such that $(\gamma_n)_{n \in \N}$ is monotonically decreasing with  $\lim_{n\to\infty} \gamma_n=0$.
If
\[ 
\mathrm{(a) \ } \sigma_n^2 = \cO \Bigl(  \gamma_n^{-1}  t_n^{-\frac {2\beta}{2\beta-1}} \Bigr),   \ \mathrm{(b) \ }  \sigma_n^q = \cO \Bigl( \gamma_n^{1-q} t_n^{-\frac{2\beta+q-1}{2\beta-1}} \Bigr) \text{  \  and \  } \ \mathrm{(c) \ }   \gamma_n = \cO \Bigl( t_n^{-\frac 1{2\beta-1}} \Bigr),
\]
then for every Lipschitz function $\nabla f:\R^d\to\R^d$, $U\subset \R^d$ bounded and $\delta>0$ for sufficiently large constant $C_w$ there exist constants $\kappa$ and $C_v$ such that for 
\begin{align*}
	v_n=C_v  t_n^{-\frac 1{2\beta-1}}\text{ \ and \ }w_n=C_w  t_n^{-\frac 1{2\beta-1}}
\end{align*}
the main and additional assumptions of Proposition~\ref{le:loja} are satisfied for large $N$. Moreover, for this choice of parameters, one can replace in Proposition~\ref{le:loja}, inequality~(\ref{eq:result})  by
\begin{align*}
	\E\bigl[\1_{\IB_n} (f(X_n)-f(x_0)) \bigr]\le \Bigl(1+ (y_0+8) \frac {C_v}{R} \Bigl( (2\beta-1)\eta R^{2\beta-1}+\frac {1}{t_N} \Bigr)^{\frac{1}{2\beta-1}} \Bigr) \Phi^{(R)}_{t_n-t_N}.
\end{align*}
\end{prop}

\begin{proof}
By choice of $(v_n)_{n \in \N}$, we get with the Taylor formula 
\begin{align} \label{eq:9909}
	-\frac{\Delta v_n}{v_n} \sim  \frac{1}{2\beta-1} \gamma_n  t_n^{-1} = \cO \Bigl( \gamma_n v_n^{2\beta-1} \Bigr)
\end{align}
so that (\ref{eq235}) is satisfied for large $\kappa $ for all large $n$.	
Note that~(\ref{eq835}) is satisfied for $v_n=C_v  t_n^{-1/(2\beta-1)}$ for a sufficiently large $C_v$ for all large $n$ if $(\gamma_n\sigma_n)^2$ and $(w_n^{-(q-1)}(\gamma_n\sigma_n)^q)$ are both of order $\cO(\gamma_n t_n^{-2\beta /(2\beta-1)})$. Elementary computations show that the first, resp.\  second term is of order $\cO(\gamma_n t_n^{-2\beta /(2\beta-1)})$ iff condition (a), resp.\  (b) holds. 
For the second inequality in~(\ref{eq:332211}), note that $w_n \le v_n $ if $C_v \ge C_w$ and $w_n\le 1$ for large enough $n$, since (c) implies that $t_n=\sum_{i=1}^n \gamma_i \to \infty$.
The first inequality  in~(\ref{eq:332211}) holds for large $n$ since
$
t_n=  t_{n-1}+\gamma_n\sim  t_{n-1}.
$
Moreover, the third inequality holds for large $n$ since $\gamma_n\to0$.
The rest follows since by definition of $\Phi$ and $(v_n)_{n \in \N}$ one has for $n\ge N$
\begin{align*}
	v_n\le \frac {C_v}{R} \Bigl( (2\beta-1)\eta R^{2\beta-1}+\frac {1}{t_N} \Bigr)^{1/(2\beta-1)}  \Phi^{(R)}_{t_n-t_N}.
\end{align*}
\end{proof}

\begin{rem}\label{rem:253}
\begin{enumerate}
	\item[(i)] In the case $\gamma_n= C_\gamma n^{-\gamma}$ and $\sigma_n = C_\sigma n^{\sigma}$ for $\gamma \in (1/2,1)$, $\sigma \in \R$ and $C_\gamma, C_\sigma >0$ assumptions (a), (b) and (c) are satisfied if
	\[ 	\mathrm{(a') \ } \  2\sigma\le \frac{(4\beta-1)\gamma-2\beta}{2\beta-1},   \ \mathrm{(b') \ }  \ 		\sigma  \le \frac{2\beta  \gamma -1}{2\beta-1}-\frac 1q \text{  \  and \  } \ \mathrm{(c') \ } \  \gamma \ge \frac{1}{2\beta}.
	\]
	In that case, we have $t_n \sim \frac{C_\gamma}{1-\gamma} n^{1-\gamma}$.
	\item[(ii)] In the case $\gamma_n= C_\gamma n^{-1}$ and $\sigma_n = C_\sigma n^{\sigma}$ for $\sigma \in \R$ and $C_\gamma, C_\sigma >0$ assumptions (a), (b) and (c) are satisfied if
	$\sigma<\frac 12$ and $q \ge 2$. In that case, $t_n \sim  C_\gamma\log(n)$.
\end{enumerate}
\end{rem}

\subsection{Bounding the drift term in the case where no dropdown occurs}\label{sec32}

\begin{prop}\label{prop83256} Assume all assumptions of Proposition~\ref{le:loja}.
Additionally  assume that $(-\Delta v_n/\gamma_n)_{n > N}$ is decreasing and $(\gamma_n/(-\Delta \bar\xi_n))^{1/2} \bar \xi_n\to 0$, with $\bar \xi_n := \Phi^{(R)}_{t_n-t_N}+(y_0+8) v_n$. Then 
\begin{align*}
	\sqrt{\frac{1-\delta'}{C^3}}&\sum_{n=N+1}^\infty \gamma_n \E[\1_{\IB_{n-1}} |\Gamma_{n}|] \le \gamma_{N+1} \sqrt{\eta} R^\beta +\sqrt \eta \int _0^\infty (\Phi_u^{(R)})^{\beta}\, du \\
	&   +\frac1{\sqrt{y_0+8}}\sum_{n=N+1}^\infty\sqrt{\frac {\gamma_n} {-\Delta v_n}}\bigl(-(y_0+8) \Delta  v_n+\gamma_n v_n^{2\beta}+8v_n \P(\IB_{n-1}\backslash \IB_n)\bigr).
\end{align*}
\end{prop}

\begin{proof}
Using that $w_n\le v_n$ we get with~(\ref{eq985761}) that
\begin{align*}
	(1-\delta')C^{-1} \gamma_n \E[\1_{\IB_{n-1}} |\nabla f(X_{n-1})|^2]\le \xi_{n-1}-\xi_n+\gamma_n v_n^{2\beta}+8v_n \P(\IB_{n-1}\backslash \IB_n)
\end{align*}
and, hence,
\begin{align}\begin{split}
		\sqrt{(1-\delta')C^{-1}} \gamma_n  \E[\1_{\IB_{n-1}} |\nabla f(X_{n-1})|]&\le \sqrt{\gamma_n} \sqrt{(1-\delta')C^{-1}\gamma_n  \E[\1_{\IB_{n-1}} |\nabla f(X_{n-1})|^2])}\\
		&\le \sqrt{\gamma_n} \sqrt{-\Delta \xi_n+\gamma_n v_n^{2\beta}+8v_n \P(\IB_{n-1}\backslash \IB_n)}.\label{eq24785}
\end{split}\end{align}
With the Cauchy-Schwarz inequality we get that
\begin{align}\begin{split}\label{eq8563}
		&\sum_{n=N+1}^\infty\sqrt{\gamma_n} \sqrt{-\Delta \xi_n+\gamma_n v_n^{2\beta}+8v_n \P(\IB_{n-1}\backslash \IB_n)}\\
		\le \Bigl(\sum_{n=N+1}^\infty & \sqrt{\gamma_n (-\Delta \bar\xi_n)} \Bigr)^{1/2} \Bigl(\sum_{n=N+1}^\infty\sqrt{\frac {\gamma_n} {-\Delta\bar\xi_n}}(-\Delta \xi_{n}+\gamma_n v_n^{2\beta}+8v_n \P(\IB_{n-1}\backslash \IB_n))  \Bigr)^{1/2},
\end{split}	\end{align}
with $\bar \xi_n := \Phi^{(R)}_{t_n-t_N}+(y_0+8) v_n$. Recall that $\xi_n\le \bar \xi_n$  by Proposition~\ref{le:loja}.

We use partial summation to get an estimate for~(\ref{eq8563}). By assumption, the sequence $(-\Delta v_n/\gamma_n)_{n > N}$ is decreasing. Moreover, for $\phi_n:=\Phi_{t_n-t_N}^{(R)}$ we get with the mean value theorem that there exists $u_n\in (t_{n-1}-t_N,t_{n}-t_N)$ with $\frac{-\Delta \phi_n}{\gamma_n}=-\dot\Phi^{(R)}_{u_n}$ so that $(\frac{-\Delta \phi_n}{\gamma_n})_{n> N}$ is also decreasing.  Altogether, we thus get that  $(a_n)_{n > N}:=((\gamma_n/(-\Delta \bar\xi_n))^{1/2})_{n> N}$ is increasing so that its differences are non-negative.
We conclude with partial  summation  that
\begin{align*}
	\sum_{n=N+1}^\infty a_n (-\Delta \xi_n) &=a_{N+1}\xi_{N}+\sum_{n=N+1}^\infty \Delta a_{n+1}  \,\xi_{n}\\
	&\le  a_{N+1}\bar \xi_{N} +\sum_{n=N+1}^\infty \Delta a_{n+1}  \,\bar \xi_{n}=\sum_{n=N+1}^\infty a_n (-\Delta \bar \xi_n),
\end{align*}
where we used that $a_{n+1}\bar \xi_n$ tends to zero. 
Consequently, we get with~(\ref{eq24785}) and~(\ref{eq8563}) that
\begin{align*}
	\sqrt{(1-\delta')C^{-1}} &\sum_{n=N+1}^\infty \gamma_n \E[\1_{\IB_{n-1}} |\nabla f(X_{n-1})|] \\
	&\le  \sum_{n=N+1}^\infty\sqrt{\frac {\gamma_n} {-\Delta \bar\xi_n}}\bigl(-\Delta \bar \xi_n+\gamma_n v_n^{2\beta}+8v_n \P(\IB_{n-1}\backslash \IB_n)\bigr) .
\end{align*}
Note that 
\begin{align*}
	\sqrt{\frac{\Phi_{0}^{(R)}-\Phi_{t_{N+1}-t_N}^{(R)}}{\gamma_{N+1}}}\le \sqrt{-\dot\Phi_0^{(R)}}\le \sqrt \eta R^{\beta}
\end{align*}
and, 
for $n>N+1$,
\begin{align*}
	\gamma_n\sqrt{\frac{\Phi_{t_{n-1}-t_N}^{(R)}-\Phi_{t_n-t_N}^{(R)}}{\gamma_n}}\le \gamma_n \sqrt{-\dot\Phi^{(R)}_{t_{n-1}-t_N}}\le \sqrt{\eta} \int_{t_{n-2}-t_N}^{t_{n-1}-t_N} (\Phi^{(R)}_{u})^{\beta}\, du,
\end{align*}
where we used that $-\dot\Phi^{(R)}$ and $(\gamma_n)_{n > N}$ are monotonically decreasing. 
Hence,	using that
$-\Delta\bar \xi_n= \Phi_{t_{n-1}-t_N}^{(R)}-\Phi_{t_n-t_N}^{(R)}+(y_0+8) (-\Delta v_n)$  we get that
\begin{align*}
	\sum_{n=N+1}^\infty\gamma_n \sqrt{ \frac{-\Delta \bar\xi_n}{\gamma_n}} \le& \gamma_{N+1} \sqrt{\eta } R^\beta +\sqrt\eta \int_0^\infty (\Phi_u^{(R)})^\beta\, du\\
	&+ \sqrt{y_0+8}\sum_{n=N+1}^\infty\sqrt{\gamma_n(-\Delta v_n)}.
\end{align*}
Consequently,
\begin{align*}
	&\sqrt{(1-\delta')C^{-1}}\sum_{n=N+1}^\infty \gamma_n \E[\1_{\IB_{n-1}} |\nabla f(X_{n-1})|] \le \gamma_{N+1} \sqrt{\eta} R^\beta +\sqrt \eta \int _0^\infty (\Phi_u^{(R)})^{\beta}\, du \\
	&   +\frac1{\sqrt{y_0+8}}\sum_{n=N+1}^\infty\sqrt{\frac {\gamma_n} {-\Delta v_n}}\bigl(-(y_0+8) \Delta  v_n+\gamma_n v_n^{2\beta}+8v_n \P(\IB_{n-1}\backslash \IB_n)\bigr)
\end{align*}
\end{proof}

\subsection{Technical analysis of lower dropdowns}\label{sec33}

Roughly speaking, the following two lemmas will later be used to show that for a certain critical level of the objective function, a lower dropdown (in the sense of the previous two propositions) entails that the SGD's target value converges to a value strictly below the respective critical level with high probability.

\begin{lemma}\label{le:mart}
Let $(M_n)_{n \in \N}$ be a $L^2$-martingale started in zero. Then for every $\kappa>0$
\begin{align*}
	\P\Bigl(\sup_{\ell\in\N} (M_\ell-\langle M\rangle_\ell) \ge \kappa\Bigr)\le \frac 4{\kappa^2}+\sum_{n\in\N_0} \frac{2^{n+3}}{(2^n+\kappa)^2}=:\phi(\kappa).
\end{align*}
In particular, there exists for every $\eps>0$ a $\kappa>0$ such that the above right-hand side is smaller than $\eps$.
\end{lemma}

\begin{proof}
For $n \in \N_0$, let $T_n=\inf\{\ell\in\N: \langle M\rangle _{\ell+1}>2^n\}$. Then,
\begin{align*}
	\sup_{\ell=T_n+1,\dots,T_{n+1}} (M_\ell-\langle M\rangle_\ell)\le  \sup_{\ell=1,\dots,T_n} M_\ell -2^n
\end{align*}
and
\begin{align*}
	\sup_{\ell=1,\dots,T_{0}} (M_\ell-\langle M\rangle_\ell)\le \sup_{\ell=1,\dots,T_{0}} M_\ell.
\end{align*}
We use Doob's $L^2$-inequality to deduce that
\begin{align*}
	\P\Bigl({ \sup_{\ell=1,\dots,T_{n+1}} M_\ell \ge 2^{n}+\kappa }\Bigr)\le 4  (2^{n}+\kappa)^{-2} \E[M_{T_{n+1}}^2] \le    \frac{2^{n+3}}{(2^{n}+\kappa)^{2}}
\end{align*}
and
$
\P(\sup_{\ell=1,\dots,T_{0}} M_\ell\ge \kappa)\le \frac 4{\kappa^2}.
$
Therefore,
\begin{align*}
	\P\Bigl(\sup_{\ell\in\N} (M_\ell-\langle M\rangle_\ell) \ge \kappa\Bigr)\le \frac 4{\kappa^2}+\sum_{n\in\N_0} \frac{2^{n+3}}{(2^n+\kappa)^2}.
\end{align*}
\end{proof}

\begin{lemma}\label{le:47236} Let $N\in\N$, $ \delta, C>0$, $(\sigma_n)_{n>N}$ be a sequence of positive reals and let $(\IA_n)_{n\ge N}$ be a sequence of $(U,\delta,C, (\sigma_n),2)$-compatible events. Moreover, suppose that $(\gamma_n\sigma_n^2)_{n>N}$ is decreasing and that for $\delta'\in(0,1)$, $2 C^3\|\nabla f\|_{\mathrm{Lip}(U^{2\delta})}\gamma_n\le \delta'$.  Then one has for $T>0$ that \begin{align*}
	\P\Bigl(\sup_{n> N}& \1_{\IA_{n-1}\cap\{\gamma_n|D_n|<\delta\}} (f(X_n)-f(X_N)) \ge 2T\Bigr) \\
	&\le \phi\Bigl(\frac{1-\delta'}{C \gamma_{N+1}\sigma_{N+1}^2}T\Bigr)+2 \|\nabla f\|_{\mathrm{Lip}(U^{2\delta})}  \sum_{\ell=N+1}^\infty (\gamma_\ell \sigma_\ell)^2\, \frac1T,
\end{align*}
where $\phi$ is as in Lemma~\ref{le:mart}. 
\end{lemma}

\begin{proof}
By (\ref{eq:92736453}) and (\ref{eq:445344}) we have on $\IA_{n-1}\cap\{\gamma_n |D_n|\le \delta\}$
\begin{align*}
	f(X_n)\le &f(X_{n-1})- \gamma_n((1-\delta')C^{-1} |\nabla f(X_{n-1})|^2+\langle \nabla f(X_{n-1}),D_n\rangle)\\
	&+2 \|\nabla f\|_{\mathrm{Lip}(U^{2\delta})}\gamma_n^2 |D_n|^2.
\end{align*}
For $n \ge N$, we let 
\begin{align*}
	\Xi_n=-\sum_{\ell=N+1}^n \1_{\IA_{\ell-1}} \gamma_\ell \big((1-\delta')C^{-1}|f(X_{\ell-1})|^2+\langle f(X_{\ell-1}),D_\ell\rangle \big)
\end{align*}
and 
$
\Xi_n'=2\|\nabla f\|_{\mathrm{Lip}(U^{2\delta})} \sum_{\ell=N+1}^n \1_{\IA_{\ell-1}}\gamma_\ell^2 |D_\ell|^2
$
and observe that on $\IA_{n-1}\cap \{\gamma_n|D_n|\le \delta\}$
\begin{align*}
	f(X_n)-f(X_N)\le \Xi_n+\Xi_n'.
\end{align*}
Next, we deduce an estimate for the supremum of the process $(\Xi_n)_{n > N}$.
In terms of the martingale $M_n=-\sum_{\ell=N+1}^n \1_{\IA_{\ell-1}} \gamma_\ell \langle \nabla f(X_{\ell-1}), D_\ell\rangle$ we have
\begin{align*}
	\langle M\rangle_n=\sum_{\ell=N+1}^n \gamma_\ell^2  \1_{\IA_{\ell-1}} \E[\langle \nabla f(X_{\ell-1}), D_\ell\rangle^2|\cF_{\ell-1}] \le \sum_{\ell=N+1}^n \gamma_\ell^2  \1_{\IA_{\ell-1}} |\nabla f(X_{\ell-1})|^2 \sigma_\ell^2
\end{align*}
Using that $(\gamma_n \sigma_n^2)_{n > N}$ is monotonically decreasing we deduce that
\begin{align*}
	\langle M\rangle _n\le \gamma_{N+1}\sigma_{N+1}^2 \sum_{\ell=N+1}^n \1_{\IA_{\ell-1}}\gamma_\ell|\nabla f(X_{\ell-1})|^2.
\end{align*}
Consequently,
\begin{align*}
	\Xi_n \le M_n- \frac {1-\delta'}{C \gamma_{N+1} \sigma_{N+1}^2}\langle M\rangle _n=a\Bigl( \frac 1{a}  M_n-\langle \frac 1{a} M\rangle _n\Bigr),
\end{align*}
for $a:= \frac {C \gamma_{N+1} \sigma_{N+1}^2}{1-\delta'}$. 
With Lemma~\ref{le:mart} we get that
\begin{align}\label{eq98664}
	\P\Bigl(\sup_{n> N} \Xi_n\ge T\Bigr) \le \P\Bigl(\sup_{n > N}  \frac 1{a}  M_n-\langle \frac 1{a} M\rangle _n\ge \frac Ta\Bigr)\le \phi\Bigl(\frac {T}a\Bigr).
\end{align}
Conversely,
\begin{align*}
	\E\Bigl[\sup_{n>N} \Xi_n'\Bigr]\le 2 \|\nabla f\|_{\mathrm{Lip}(U^{2\delta})} \E\Bigl[ \sum_{\ell=N+1}^\infty \1_{\IA_{\ell-1}}\gamma_\ell^2 |D_\ell|^2\Bigr]\le 2 \|\nabla f\|_{\mathrm{Lip}(U^{2\delta})}  \sum_{\ell=N+1}^\infty (\gamma_\ell \sigma_\ell)^2 .
\end{align*}
and by the Markov inequality
\begin{align*}
	\P\Bigl(\sup_{n> N} \Xi_n'\ge T\Bigr)\le 2 \|\nabla f\|_{\mathrm{Lip}(U^{2\delta})}  \sum_{\ell=N+1}^\infty (\gamma_\ell \sigma_\ell)^2\,\frac1T.
\end{align*}
In combination with~(\ref{eq98664}) we obtain the result.
\end{proof}

\subsection{Satisfiability of the assumptions}\label{sec34}

\begin{prop} \label{prop:assu}
Let  $\gamma \in (\frac 12, 1]$, $\sigma\in\R$ and $q \ge 2$ with
\begin{align}\label{eq357}
	3\gamma-2\sigma>2  , \qquad  q > \frac{1}{2\gamma-\sigma-1}
\end{align}
and let, for positive constants $C_\gamma$ and $C_\sigma$,	$(\gamma_n)_{n \in \N} =(C_\gamma n^{-\gamma})_{n \in \N}$  and $(\sigma_n)_{n \in \N}=(C_\sigma n^{\sigma})_{n \in \N}$.
Suppose that $f$ satisfies a \loja-inequality on a bounded set $U$. Then there exists $\beta\in(\frac 12,1)$ such that for $(v_n)_{n \in \N}=(C_v t_n^{-\frac 1{2\beta-1}})_{n \in \N}$ and $(w_n)_{n \in \N}=(C_wt_n^{-\frac 1{2\beta-1}})_{n \in \N}$ with   $C_w, C_v>0$ sufficiently large, one has for sufficiently large $N$ and all compatible events $(\IB_n)_{n\ge N}$ in the sense of Prop.~\ref{le:loja} that
\begin{align*}
	\sum_{n=N+1}^\infty \gamma_n \,\E[\1_{\IB_{n-1}}|\Gamma_{n}|]<\infty.
\end{align*}
Moreover, as $N\to\infty$, $w_N/(\gamma_{N+1}\sigma_{N+1}^2)\to\infty$ and $\sum_{\ell=N+1}^\infty(\gamma_\ell\sigma_\ell)^2 w_N^{-1}\to0$.
\end{prop}

\begin{proof}
First note that since $U$ is bounded we can conclude that validity of a \loja-inequality for a $\beta_0\in(\frac12,1)$ entails a \L ojasiewicz-inequality for every $\beta\in[\beta_0,1)$.
By Proposition~\ref{prop583} and Remark~\ref{rem:253}, one can choose for every large $C_w$, appropriate constants  $C_v$ and $\kappa$ such that all assumptions of Prop.~\ref{le:loja} are satisfied, if 
\[ 	\mathrm{(a') \ } \  2\sigma\le \frac{(4\beta-1)\gamma-2\beta}{2\beta-1},   \ \mathrm{(b') \ }  \ 		\sigma  \le \frac{2\beta  \gamma -1}{2\beta-1}-\frac 1q \text{  \  and \  } \ \mathrm{(c') \ } \  \gamma \ge \frac{1}{2\beta}\]
By sending $\beta$ to one, one easily verifies by using~(\ref{eq357}) and $\gamma>\frac12$, that (a'), (b') and (c') are satisfied for a $\beta\in[\beta_0,1)$. In view of Prop.~\ref{prop83256}, it suffices to verify the additional condition and finiteness of the series appearing in the latter proposition.

We still need to consider the  additional conditions imposed in Proposition~\ref{prop83256}. First we show that the sequence $(-\Delta v_n/\gamma_n)_{n \ge 2}$ is eventually decreasing.
One has 
\begin{align*}
	\frac {-\Delta v_n}{\gamma_n}-\frac {-\Delta v_{n+1}}{\gamma_{n+1}} =\frac {-\Delta v_n \gamma_{n+1}+\Delta v_{n+1}\gamma_n}{\gamma_n\gamma_{n+1}},
\end{align*}
\begin{align*}
	-\Delta v_n= \frac{C_v}{2\beta-1} t_n^{-\frac{2\beta}{2\beta-1}} \gamma_n + \frac{C_v \beta}{(2\beta-1)^2} t_n^{-\frac{4\beta-1}{2\beta-1}} \gamma_n^2+ o(t_n^{-\frac{4\beta-1}{2\beta-1}}\gamma_n^2)
\end{align*}
and
\begin{align*}
	-\Delta v_{n+1}= \frac{C_v}{2\beta-1} t_n^{-\frac{2\beta}{2\beta-1}} \gamma_{n+1} -  \frac{C_v \beta}{(2\beta-1)^2} t_n^{-\frac{4\beta-1}{2\beta-1}} \gamma_{n+1}^2+ o(t_n^{-\frac{4\beta-1}{2\beta-1}}\gamma_n^2),
\end{align*}
so that 
\begin{align*}
	-\Delta v_n \gamma_{n+1} + \Delta v_{n+1}\gamma_n= \frac{C_v \beta}{(2\beta-1)^2} t_n^{-\frac{4\beta-1}{2\beta-1}}\gamma_n \gamma_{n+1}(\gamma_n+\gamma_{n+1})+o(t_n^{-\frac{4\beta-1}{2\beta-1}}\gamma_n^3)
\end{align*}
which is positive for large $n$.

Further, using (\ref{eq:9909}) we get $(\gamma_n/(-\Delta v_n))^{1/2}v_n = \cO ( (v_n^{-1} t_n)^{1/2} v_n ) = \cO ( v_n^{1-\beta})$ and with $v_n \sim C_v\bigl( (2\beta-1)\eta \bigr)^{1/(2\beta-1)} \Phi_{t_n-t_N}^{(R)} $ we have 
$
(\gamma_n/(-\Delta\bar \xi_n))^{1/2}\bar \xi_n \to0.
$

Next, we show that the  series appearing in Proposition~\ref{prop83256} is finite. Indeed, we have $\sqrt{\gamma_n(-\Delta v_n)}\sim  \sqrt{\frac{C_v}{2\beta-1}}  \gamma_n t_n^{-\frac{\beta}{2\beta-1}}$ and 
\begin{align*}
	\sqrt{\gamma_n/(-\Delta v_n)} \gamma_n v_n^{2\beta}\sim C_v^{2\beta-1/2}\sqrt{2\beta-1} \gamma_n t_n^{-\frac{\beta}{2\beta-1}},
\end{align*}
so that for $\gamma<1$  summability follows with $\beta/(2\beta-1)>1$ for all $\beta \in (1/2,1)$ and in the case $\gamma=1$ summability follows from Cauchy's condensation test.
Moreover, as stated before $(\gamma_n/(-\Delta v_n))^{1/2} v_n= \cO( v_n^{1-\beta})$ tends to zero.

It remains to discuss the asymptotic statements in the proposition. The term $w_N/(\gamma_{N+1}\sigma_{N+1}^2)\sim \frac{C_w}{C_\gamma C_\sigma^2} (t_n)^{-\frac{1}{2\beta-1}} n^{\gamma-2\sigma} $ tends to infinity iff $\beta>\frac{1-2\sigma}{2\gamma-4\sigma}$. As consequence of~(\ref{eq357}) the right-hand side is smaller than $3/4$ and for sufficiently large $\beta$, $w_N/(\gamma_{N+1}\sigma_{N+1}^2)$ tends to infinity. Moreover, $\sum_{\ell=N+1}^\infty (\gamma_\ell\sigma_\ell)^2= \cO( N^{-2\gamma+2\sigma+1})$ is of order $o(w_n)$ iff $\beta>\frac{\gamma-2\sigma}{4\gamma-4\sigma-2}$ with the right hand side being strictly smaller than one.
\end{proof}

\subsection{Properties of \loja-functions}\label{sec35}
We summarise some properties of \L ojasiewicz-functions. Note that the proofs only use that $f$ is $C^1$ but not that the differentials are locally Lipschitz continuous. A similar statement for a uniform \loja-inequality around a compact level set can be found in~\cite[Lemma 1]{attouch2009convergence}.

\begin{lemma}\label{prop:loja1}Let $f:\R^d\to\R$ be a \loja-function and $K\subset \R^d$ be an arbitrary compact set. 
\begin{enumerate}
	\item[(i)] The set of \emph{critical levels}
	\begin{align*}
		\cL_K=\{f(x):x\in \nabla f^{-1}(\{0\})\cap K\}
	\end{align*}
	is finite so that $f$ has at most a countable number of critical levels.
	\item[(ii)] For every critical level $\ell\in \cL_K$ there exists an open neighbourhood $U\supset  f^{-1}(\{\ell\})\cap K$, $\CL>0$, $\beta\in[\frac 12 ,1)$ such that for every $y\in U$
	\begin{align*}
		|\nabla f(y)|\ge \CL |f(y)-\ell|^\beta.
	\end{align*}
	\item[(iii)] For a neighbourhood as in (ii), there exists $\eps>0$ such that
	\begin{align*}
		f^{-1}((\ell-\eps,\ell+\eps))\cap K\subset U.
	\end{align*}
\end{enumerate}
\end{lemma}

\begin{proof}(i): 	Let $K\subset U$ be a compact set. For every $x\in \nabla f^{-1}(\{0\})$ we can choose an open neighbourhood $U_x$ of $x$ on which the \loja-inequality holds for parameters $\CL^{(x)}>0$ and $\beta^{(x)}\in[\frac 12,1)$. By compactness of $\nabla f^{-1}(\{0\})\cap K$ we conclude that
$
\nabla f^{-1}(\{0\})\cap K \subset \bigcup_{x\in \nabla f^{-1}(\{0\})\cap K} U_x
$
has a finite cover, say $\bigcup_{x\in \cX} U_x$. We show that
$\cL_K=F(\cX)$.
Indeed, every $y\in \nabla f^{-1}(\{0\})\cap K$ lies in a neighbourhood $U_x$ with $x\in\cX$ and one has
\begin{align*}
	0=|\nabla f(y)|\ge \CL^{(x)} |f(y)-f(x)|^{\beta^{(x)}} 
\end{align*}
so that $f(y)=f(x)\in f(\cX)$.\smallskip

(ii): Let $x\in K\cap f^{-1}(\{\ell\})$. If $x$ lies in $\nabla f^{-1}(\{0\})$, then $x$ admits an open and bounded neighbourhood $U_x$ on which the \loja-inequality holds with appropriate parameters $\CL^{(x)}>0$ and $\beta^{(x)}\in[\frac 12,1)$. If $x$ does not lie in $\nabla f^{-1}(\{0\})$, then $|\nabla f(x)|>0$ and by continuity we can  choose a neighbourhood $U_x$ so that for every $y\in U_x$ , $|\nabla f(y)|>|\nabla f(x)|/2$ and $|f(y)-f(x)|^{1/2}<|\nabla f(x)|/2$. Thus the \loja-inequality holds on $U_x$ with parameters $1$ and $1/2$.
By compactness of $K\cap f^{-1}(\{\ell\})$ there exists a finite set $\cX\subset K\cap f^{-1}(\{\ell\})$ with
$
f^{-1}(\{\ell\}) \cap K \subset \bigcup_{x\in \cX} U_x.
$
We let $\beta:=\max_{x\in \cX} \beta^{(x)}$ and
assuming that on $U:=\bigcup_{x\in \cX} U_x$, $|f-\ell|$ is bounded by $R$ we conclude that for every $y\in U$ there exists $x\in \cX$ with $y\in U_x$ so that
\begin{align*}
	|\nabla f(y)|\ge \CL^{(x)} |f(y)-\ell|^{\beta^{(x)}}\ge \frac {\CL^{(x)}}{R^{\beta-\beta^{(x)}}} |f(y)-\ell|^{\beta}
\end{align*}
and the \loja-inequality holds on $U$ with parameters $\min_{x\in\cX}{\CL^{(x)}}/{R^{\beta-\beta^{(x)}}}$ and  $\beta$.
\smallskip

(iii): Suppose that for every $\eps>0$, $K\cap f^{-1}((\ell-\eps,\ell+\eps))\not \subset U$. Then we can pick a $K\cap U^c$-valued sequence $(x_n)_{n \in \N}$ with
$
|f(x_n)-\ell|\le \frac 1n
$.
Since $K$ is compact we can assume without loss of generality that $(x_n)_{n \in \N}$ converges (otherwise we choose an appropriate subsequence). Then by continuity of $f$, $f(x_n)\to f(x)=\ell$ and $x\in K\cap f^{-1}(\{\ell\})$. But this entails that all but finitely many of the entries of $(x_n)_{n \in \N}$ have to lie in $U$ since $U$ is an open neighbourhood of $ f^{-1}(\{\ell\})\cap K$ causing a contradiction.
\end{proof}

\subsection{Proof of Theorem~\ref{theo2}}\label{sec36}
Let $K$ be a compact set  and $\Upsilon\in \cL_K$ a critical level of $f$ on $K$. 
By Lemma~\ref{prop:loja1} we can choose an open and bounded set $U\supset f^{-1}(\{\Upsilon\})\cap K$ and parameters $\CL_0>0$ and $\beta_0\in(\frac 12,1)$ such that $f$ admits the \loja-inequality on $U$, i.e. that for all $y\in U$
\begin{align*}
|\nabla f(y)|\ge \CL_0 |f(y)-\Upsilon|^{\beta_0}.
\end{align*}
Again by Lemma~\ref{prop:loja1} we can pick $\eps>0$ with
$
U\supset  f^{-1}((\Upsilon-\eps,\Upsilon+\eps)).
$
Based on parameters  $N\ge N_0$ and $C>0$ we consider events $(\IA_n)_{n\ge N}$ given by 
\begin{align*}
\IA_N=\bigl\{X_N\in U, |\nabla f(X_N)|^2 \le &C \langle \nabla f(X_N), -\Gamma_{N+1} \rangle  , \,   |\Gamma_{N+1}|\le C |\nabla f(X_N)|, \\
&\E[|D_{N+1}|^q|\cF_{N}]\le C\sigma_n^q\text{ and } \E[D_{N+1}|\cF_{N}]=0\bigr\}
\end{align*}
and, for $n>N$,
\begin{align*}
\IA_n=\IA_{n-1}\cap\bigl\{X_n & \in U, |\nabla f(X_n)|^2 \le C \langle \nabla f(X_n), -\Gamma_{n+1} \rangle  , \,   |\Gamma_{n+1}|\le C |\nabla f(X_n)| \\
& \gamma_n|D_n|< \delta,\E[|D_{n+1}|^q|\cF_{n}]\le C\sigma_n^q\text{ and } \E[D_{n+1}|\cF_{n}]=0\bigr\}.
\end{align*}
Provided that $N$ is sufficiently large such that, on $\IA_n$, $\gamma_{n+1} |\Gamma_{n+1}| \le \gamma_{n+1} C \|\nabla f\|_{L^\infty(U)} \\ \le \delta$ for all $n\ge N$, the family $(\IA_n)_{n\ge N}$ is $(U,\delta,C,  (C\sigma_n), q)$-compatible.

Now  for every $\beta \in[\beta_0,1)$ there exists a $\CL>0$ such that $f$ also admits the \loja-inequality with parameters $\beta$ and $\CL$ on $U$ and we can choose 
a quadruple $(N_0, 
\beta, C_v, C_w)$ which satisfies the statement of Proposition~\ref{prop:assu} for $\sigma_n=Cn^\sigma$.	
We assume without loss of generality that $\Upsilon=0$ mentioning that we could consider $\bar f=f-\Upsilon$ instead of $f$ and fix a sufficiently large $N \ge N_0$ so that $(\IA_n)_{n\ge N}$ is $(U,\delta,C,  (C\sigma_n), q)$-compatible.

Additionally, we let $\IB_N=\IA_N \cap\{-w_N \le f(X_N)\}$ and for $n>N$
\begin{align*}
\IB_n=\IB_{n-1}\cap \IA_{n-1} \cap \bigl\{\gamma_n|D_n|< \delta,
-w_n \le f(X_n)\bigr\}.
\end{align*}
We note that $(\IB_n)_{n\ge N}$ satisfies a representation as in Lemma~\ref{le:loja} for $(\IA_n')_{n\ge N}$ given by 
$\IA_N'=\IA_N$
and, for $n>N$,
\begin{align*}
\IA_n'=\IB_n\cap\bigl\{X_n\in U,  |\nabla f(X_n)|^2 &\le C \langle \nabla f(X_n), -\Gamma_{n+1} \rangle  , \,   |\Gamma_{n+1}|\le C |\nabla f(X_n)|, \\
&\E[|D_{n+1}|^q|\cF_{n}]\le C\sigma_n^q\text{ and } \E[D_{n+1}|\cF_{n}]=0\bigr\}.
\end{align*}
In particular, one has $\IA_n'\subset \IA_n$ and $(\IA_n')_{n \ge N}$ is $(U,\delta, C, (C\sigma_n), q)$-compatible if this is the case for $(\IA_n)_{n \ge N}$.
Note that  for $\IB_\infty = \bigcap_{n \ge N} \IB_n$ (and respectively for $\IA_\infty$ and $\IA_\infty'$) we have that
\begin{align*}
\IB_\infty=\IA_\infty'=\IA_\infty \cap\{F(X_n)\ge -w_n\text{ for all }n\ge N\}.
\end{align*}
By  Propositions~\ref{prop83256} and~\ref{prop:assu}, one has that on $\IA_\infty'$, almost surely, $\sum_{n>N} \gamma_n |\Gamma_{n}|<\infty$.
Moreover,  the $L^2$-martingale $(\bar M_n)_{n\ge N}=(\sum_{\ell=N+1}^n \gamma_\ell \1_{\IA_{\ell-1}} D_\ell)_{n \ge N}$ converges, a.s., since 
$
\E[\langle \bar M\rangle_\infty]\le \sum_{\ell=N+1}^\infty C^{2}(\gamma_\ell \sigma_\ell)^2$.
Hence, we have almost sure convergence of $(X_n)_{n \in \N_0}$ on $\IA'_\infty$.
Now consider the stopping time $T$ given by
\begin{align*}
T(\omega)=\inf\{n\ge N: f(X_n(\omega))<-w_n, \omega\in \IA_n\}.
\end{align*}
For fixed $N'\ge N$ we apply Lemma~\ref{le:47236} to estimate the probability
\begin{align*}
\P\bigl(\sup_{n\ge N'} f(X_n)-f(X_{N'})\ge w_{N'}   \big|T=N' \bigr).
\end{align*}
Note that $(\IA_{n}'')_{n\ge N'}$ with $\IA_n''=\{T=N'\}\cap \IA_n$ is $(U,\delta,C, (C\sigma_n),2)$-compatible under the conditional distribution $\P(\,\cdot\, |T=N')$. Recall that $2C^3\|\nabla f\|_{\mathrm{Lip}(U^{2\delta})}\gamma_n\le \delta'$ 
so that
\begin{align*}
\P\bigl(\sup_{n>N'}&\1_{\IA_{n-1}\cap\{\gamma_n|D_n|<\delta\}} (f(X_n)-f(X_{N'}))\ge w_{N'}   \big|T=N' \bigr)\\
&\le \phi\Bigl(\frac{1-\delta'}{2C^3 \gamma_{N'+1}\sigma_{N'+1}^2}w_{N'}\Bigr)+4 C^2\|\nabla f\|_{\mathrm{Lip}(U^{2\delta})}  \sum_{\ell=N'+1}^\infty (\gamma_\ell \sigma_\ell)^2\, \frac1{w_{N'}}=:\rho_{N'}.
\end{align*}
By Proposition~\ref{prop:assu}, $\rho_{N'}\to 0$ as $N'\to \infty$ so that also $\bar \rho_N:=\sup_{n\ge N}\rho_n$ converges to zero. We conclude that
\begin{align*}
&\P\bigl((\IA_\infty\backslash\IB_\infty) \cap \bigl\{\lim_{n\to\infty} f(X_n)=0\bigr\}\bigr) =\P\bigl(\IA_\infty\cap \bigl\{T<\infty, \lim_{n\to\infty} f(X_n)=0\bigr\}\bigr)\\
&\le \sum_{N'=N}^\infty \P(T=N') \P\Bigl(\sup_{n>N'} \1_{\IA_{n-1} \cap\{\gamma_n|D_n|<\delta\}}(f(X_n)-f(X_{N'})) \ge w_{N'} \Big|T=N'\Bigr)\\
&\le \bar \rho_N.
\end{align*}
We recall that $\IA_\infty$ and $\IB_\infty$ depend on the choice of $N$ and in the following we write $\bar \IA_N$ and $\bar\IB_N$ for the respective events.
Moreover, we denote
\begin{align*}
\bar \IM^{C, \sigma ,q}_ N:=\bigcap_{n=N}^\infty \bigl\{|\nabla f(X_n)|^2 &\le C \langle \nabla f(X_n), -\Gamma_{n+1} \rangle  , \,   |\Gamma_{n+1}|\le C |\nabla f(X_n)|, \\
&\E[|D_{n+1}|^q|\cF_{n}]\le (C\sigma_{n+1})^q \text{ and } \E[D_{n+1}|\cF_{n}]=0\bigr\},
\end{align*}
\begin{align*}
\bar \IM^{C, \sigma,q}=\bigcup _{N\in \N}\bar \IM^{C,\sigma,q}_ N, \ \IU_N:=\bigcap_{n=N}^\infty \{X_n\in U\} \text{ and } \IU :=\bigcup _{N \in \N} \IU_N.
\end{align*}
We note that $(\bar \IM^{C,\sigma,q}_ N\cap \IU_N )\backslash \bar \IA_N=\bar \IM^{C,\sigma,q}_ N\cap \IU_N \cap\{\exists n>N: \gamma_n |D_n|\ge \delta\}$ and estimate
\begin{align*}
&\P(\bar\IM^{C,\sigma,q}_ N\cap \IU_N \cap\{\exists n>N: \gamma_n |D_n|\ge \delta\})\\
&\le \sum_{n=N+1}^\infty \P\bigl( \gamma_n |D_n|\ge \delta\,\big|\,\E[|D_n|^2 |\cF_{n-1}]\le  (C\sigma_n)^2\bigr) \le \delta^{-1} \sum_{n=N+1}^\infty C^2(\gamma_n\sigma_n)^2.
\end{align*}
By Theorem~\ref{theo1}, we have a.s.\ convergence of $(f(X_n))_{n \in \N_0}$ on $\bar\IM^{C,\sigma,q}_ N\cap \IU_N$ and since $f$ has a unique critical level on $U$ we thus have $f(X_n)\to0$, a.s.,  on $\bar\IM^{C,\sigma,q}_ N\cap \IU_N$.
We conclude that
\begin{align*}
\P(\bar\IM^{C,\sigma,q}_ N\cap &\IU_N\cap\{(X_n)\text{ does not converge}\})\le \P((\bar\IM^{C\sigma ,q}_ N\cap \IU_N )\backslash \bar \IA_N)+\P(\bar \IA_N\backslash\bar \IB_N)\\
&\le \delta^{-1} \sum_{n=N+1}^\infty C^2(\gamma_n\sigma_n)^2 +\bar \rho_N \to 0, \text{ \ as }N\to\infty.
\end{align*}
Thus we obtain almost sure convergence of $(X_n)_{n \in \N_0}$ on the monotone limit  $\bar \IM^{C\sigma ,q}\cap \IU=\bigcup_{N \in \N} \bar \IM^{C,\sigma,q}_N\cap \IU_N$. Obviously, this is also true on $\IG \cap \IM^{\sigma,q}\cap \IU$ since 
\begin{align*}
\IG \cap \IM^{\sigma,q} =\bigcup_{C\in \N} \bar \IM^{C,\sigma,q}.
\end{align*}

Now note that there is only a finite number of distinct critical levels on $K$, say  $\ell_1,\dots,\ell_J$, and for each critical level $\ell_j$ $(j=1,\dots,J)$ we can choose an open \loja-neighbourhood $U _j$ as above. 
By elementary analysis it follows that for all sequences $(x_n)_{n \in \N_0}$ with
(a) all but finitely many $x_n$ lie in $K$, (b) $\lim_{n\to\infty} f(x_n)$ exists and (c) $\lim_{n\to\infty} \nabla f(x_n)=0$, one has that $(f(x_n))_{n \in \N_0}$ converges to one of the critical levels $\ell_1,\dots,\ell_{J}$ and all but finitely many of its entries are in the respective open set $U_j$.
By Theorem~\ref{theo1}, we have that in the case where $\IG \cap \IM^{\sigma,q}$ enters and all but finitely many of the entries $(X_n)_{n \in \N_0}$ lie in $K$ that all but finitely many entries of $(X_n)_{n \in \N_0}$ lie in a single set~$U_j$ and we thus obtain with the above that $(X_n)_{n \in \N_0}$ converges almost surely. Since this is true for every compact set $K$ we obtain the result.

\section{Analytic neural networks}\label{sec4}

In this section, we discuss neural networks that satisfy the assumptions of Theorem~\ref{theo2}. We use similar notation as in~\cite{PV17}. Let $\rho:\R\to\R$ denote a function, the \emph{activation function}.
We fix a \emph{depth} $L\in\N$ and the number of neurons $d_\mathrm{in}:=N_0,\dots,N_L=:d_{\mathrm{out}}$ in each of the layers and we denote by
\begin{align*}
\cP_N:=\prod_{\ell=1}^L(\R^{N_\ell\times N_{\ell-1}}\times \R^{N_\ell}),
\end{align*}
the set of all parametrizations of networks with \emph{architecture} $N=(N_0,\dots,N_L)$.
The architecture~$N$ is related to the \emph{realization map} $\cR_N:\cP_N\to C(\R^{d_\mathrm{in}},\R^{d_\mathrm{out}})$ given by
\begin{align*}
\Theta=(A_\ell,b_\ell)_{\ell=1}^L\mapsto \cR_N(\Theta)=\mathrm{Aff}_{A_L,b_L} \circ\rho^{\otimes N_{L-1}}\circ \mathrm{Aff}_{A_{L-1},b_{L-1}} \circ  \ldots \circ\rho^{\otimes N_1}\circ \mathrm{Aff}_{A_1,b_1},  
\end{align*}
where for $p,q\in\N$ for each $A\in\R^{p,q}$ and $b\in\R^p$,
$
\mathrm{Aff}_{A,b}:\R^q\to\R^p, x\mapsto Ax+b
$
and $\rho^{\otimes p}:\R^p\to \R^p$ is the mapping that applies $\rho$ component-wise.

\begin{prop}\label{prop25} If $\rho:\R\to\R$ is analytic, then
the mapping
\begin{align*}
	\cP_N\times \R^{d_\mathrm{in}} \to \R^{d_\mathrm{out}},  (\Theta,x) \mapsto \cR_N(\Theta,x)=\cR_N(\Theta)(x)
\end{align*}
is analytic.
\end{prop}

\begin{proof}Note that for each $\ell=1,\dots,L$ the mapping
\begin{align*}
	\cP_N\times \R^{N_{\ell-1}} \to \cP_N\times \R^{N_{\ell}}, (\Theta,x)\mapsto (\Theta,
	A_\ell x+b_\ell),
\end{align*}
with $(A_\ell,b_\ell)$ being the $\ell$-th entry of $\Theta$ is analytic. Moreover, for each $\ell=1,\dots,L-1$ the mapping
\begin{align*}
	\cP_N\times \R^{N_{\ell}} \to \cP_N\times \R^{N_{\ell}}, (\Theta,x)\mapsto (\Theta, \rho^{\otimes N_{\ell}}(x))
\end{align*}
is analytic. The mapping in the proposition may be written as composition of the above analytic functions. Hence, it is also analytic. 
\end{proof}

We give sufficient conditions that imply analyticity and, thus, the existence of local \loja-inequalities for the objective function $f(\Theta)=\E[\mathcal L(\mathcal R(\Theta,X),Y)]$, where $X$ and $Y$ are compactly supported random variables. In practice, the expectation is often taken as an empirical average over a finite data set, so that $X$ and $Y$ are clearly compactly supported.

\begin{theorem}\label{theo3}We assume the above setting and let $\rho:\R\to\R$ and $\cL:\R^{d_\mathrm{out}}\times \R^{d_\mathrm{out}}\to \R$ be analytic functions. Then for compactly supported $\R^{d_\mathrm{in}}$- and $\R^{d_\mathrm{out}}$-valued random variables $X$ and $Y$, the function
\begin{align*}
	f:\cP_N\to \R,\, \Theta\mapsto \E[\cL(\cR(\Theta,X),Y)]
\end{align*}
is analytic.
\end{theorem}

\begin{proof}
As consequence of Proposition~\ref{prop25} the mapping
\begin{align*}
	\cP_N\times \R^{d_\mathrm{in}}\times \R^{d_\mathrm{out}}\to\R, (\Theta,x,y)\mapsto \cL(\cR(\Theta,x),y)
\end{align*}
is analytic.
We prove that generally for $p,q\in\N$, an analytic function $G:\R^{p+q}\to \R$ and a compactly supported $\R^q$-valued random variable $Z$, the function
$
\R^p \ni x\mapsto \E[G(x,Z)]
$
is analytic. This then implies the statement of the theorem.

We will use that a function is analytic if and only if it is $C^\infty$ and satisfies locally an estimate as~(\ref{eq741}) below, see for instance Proposition 2.2.10 of \cite{krantz2002primer}.
Pick $x_0\in \R^p$ and a compact set $K\subset \R^q$ on which $Z$ is supported. Then there exists  for each $z_0\in \R^q$, an open set $U_{z_0}\supset \{(x_0,z_0)\}$, $C_z<\infty$ and $R_{z_0}\in(0,1]$ such that for all $(x,z)\in U_{z_0}$ one has for every multiindex $\mu$
\begin{align}\label{eq741}
	\Bigl| \frac{\partial^{|\mu|}G}{\partial x^\mu}(x,z) \Bigr| \le C_{z_0} \cdot \frac{\mu!}{R_{z_0}^{|\mu|}}.
\end{align}
Now $\bar U_{z_0}:=\{z\in\R^q:(x_0,z)\in U_{z_0}\}\supset \{z_0\}$ is open and there is a finite cover (a finite subset $\mathcal Z\subset \R^q$) with
$
\bigcup_{z'\in\mathcal Z}	\bar U_{z'}\supset K.
$
We let $R:=\min _{z'\in\mathcal Z} R_{z'}>0$ and $C=\max_{z'\in\mathcal Z}C_{z'}$ and note that for every $(x,z)\in\bigcup_{z'\in\mathcal Z}	U_{z'}$ one has
\begin{align*}
	\Bigl| \frac{\partial^{|\mu|}G}{\partial x^\mu}(x,z) \Bigr| \le C \cdot \frac{\mu!}{R^{|\mu|}}.
\end{align*}
Now $\bigcup_{z'\in\mathcal Z}	U_{z'}$ is open and covers $\{x_0\}\times K$ so that there exists an open neighbourhood $V$ of $x_0$ with $V\times K\subset \bigcup_{z'\in\mathcal Z}	U_{z'}$ and we observe that for every $x\in V$ and every multiindex $\mu$
\begin{align*}
	\E\Bigl[ \frac{\partial^{|\mu|}G}{\partial x^\mu}(x,Z)\Bigr]\le  C \cdot \frac{\mu!}{R^{|\mu|}}.
\end{align*}
Using that $G(\cdot,z)$ is $C^\infty$ for every $z\in\R^q$ we obtain the result.
\end{proof}

{\bf Acknowledgement.}
We thank Christoph Böhm for pointing out the relevance of  \L ojasiewicz-inequalities in the analysis of ODEs.

Funded by the Deutsche Forschungsgemeinschaft (DFG, German Research Foundation) under Germany's Excellence Strategy EXC 2044--390685587, Mathematics Münster: Dynamics--Geometry--Structure.

\bibliographystyle{alpha}
\bibliography{Lojasiewicz_inequalities}

\end{document}